\documentclass{article}

\PassOptionsToPackage{numbers, sort&compress}{natbib}
\usepackage{natbib}
\usepackage{style/mydef}

\global\long\def\EE{\mathbb{E}}%

\global\long\def\inner#1#2{\left\langle #1,#2\right\rangle }%

\global\long\def\la{\langle}%

\global\long\def\ra{\rangle}%

\global\long\def\diag{\mathrm{diag}}%

\global\long\def\EE{\mathbb{E}}%

\global\long\def\inner#1#2{\left\langle #1,#2\right\rangle }%

\global\long\def\la{\langle}%

\global\long\def\ra{\rangle}%

\global\long\def\given{\vert}%

\usepackage[verbose=true,letterpaper]{geometry}
\AtBeginDocument{
  \newgeometry{
    textheight=9in,
    textwidth=5.5in,
    top=1in,
    headheight=12pt,
    headsep=25pt,
    footskip=30pt
  }
}

\usepackage[utf8]{inputenc} 
\usepackage[T1]{fontenc}    
\usepackage{xcolor}
\definecolor{mydarkred}{rgb}{0.6,0,0}
\definecolor{mydarkgreen}{rgb}{0,0.6,0}
\usepackage[colorlinks,
linkcolor=mydarkred,
citecolor=mydarkgreen]{hyperref}
\usepackage{url}            
\usepackage{booktabs}       
\usepackage{amsfonts}       
\usepackage{nicefrac}       
\usepackage{microtype}      
\usepackage{xcolor}         
\usepackage{comment}
\usepackage{thmtools}
\usepackage{thm-restate}
\usepackage{bbm}

\newcommand{\citetMastouri}{Mastouri and Zhu et al. \citep{Mastouri2021Proximal}}

\usepackage{tikz}
\usetikzlibrary{shapes,decorations,arrows,calc,arrows.meta,fit,positioning}
\tikzset{
    -Latex,auto,node distance =1 cm and 1 cm,semithick,
    state/.style ={circle, draw, minimum width = 0.7 cm},
    point/.style = {circle, draw, inner sep=0.04cm,fill,node contents={}},
    bidirected/.style={Latex-Latex,dashed},
    el/.style = {inner sep=2pt, align=left, sloped}
}

\AtBeginDocument{
  \abovedisplayskip     =0.5\abovedisplayskip
  \abovedisplayshortskip=0.5\abovedisplayshortskip
  \belowdisplayskip     =0.5\belowdisplayskip
  \belowdisplayshortskip=0.5\belowdisplayshortskip}

\title{Deep Proxy Causal Learning and its Application to Confounded Bandit Policy Evaluation}
\author{
Liyuan Xu \\
Gatsby Unit\\
\texttt{liyuan.jo.19@ucl.ac.uk} \\
\and
Heishiro Kanagawa\\
Gatsby Unit\\
\texttt{heishiro.kanagawa@gmail.com} \\
\and
Arthur Gretton\\
Gatsby Unit\\
\texttt{arthur.gretton@gmail.com}\\
}

\begin{document}

\maketitle

\begin{abstract}

Proxy causal learning (PCL) is a method for estimating the causal effect of treatments on outcomes in the presence of unobserved confounding, using {\em proxies} (structured side information) for the confounder.
This is achieved via two-stage regression: in the first stage, we model relations among the treatment and proxies; in the second stage, we use this model to learn the effect of treatment on the outcome, given the context provided by the proxies. PCL guarantees recovery of the true causal effect, subject to identifiability conditions.
We propose a novel method for PCL, the {\it deep feature proxy variable method (DFPV)}, to address the case where the proxies, treatments, and outcomes are high-dimensional and have nonlinear complex relationships, as represented by deep neural network features.
We show that DFPV outperforms recent state-of-the-art PCL methods on challenging synthetic benchmarks, including settings involving high dimensional image data. Furthermore, we show that PCL can be applied to off-policy evaluation for the confounded bandit problem, in which DFPV also exhibits competitive performance.

\end{abstract}

\section{Introduction}
In causal learning, we aim to estimate the effect of our actions on the world. For example, we may be interested in measuring the impact of flight ticket prices on sales \citep{Wright1928, Blundell2012}, or the effect of grade retention on cognitive development \citep{Fruehwirth2016How}. 
We refer to our action as a {\em treatment}, which results in a particular {\em outcome}. 
It is often impossible to determine the effect of treatment on outcome from observational data alone, 
since the observed joint distribution of treatment and outcome can depend on a common {\em confounder} which influences both, and which might not be observed.  In our example on sales of plane tickets given a price, the two might even be {\it positively correlated} in some circumstances, such as the simultaneous increase in sales and prices during the holiday season. This does not mean that raising the price {\it causes} sales to increase. In this context, people's desire to fly is a confounder, since it affects both the number of ticket sales and the prices people are willing to accept. Thus, we need to correct the bias caused by the confounder.

One common assumption to cope with confounding bias is to assume no unobserved confounders exist \citep{Imbens2004Nonparametric}, or more generally, the {\it ignorable treatment assignment} assumption \citep{Rosenbaum1983Propensity}, which states that the treatment assignment is independent of the potential outcomes caused by the treatment, given the background data available. Although a number of methods are proposed based on this assumption \citep{Hill2011Bayesian,Johansson2016Learning,Yao2018representation}, it can be too restrictive, since it is often difficult to determine how the confounder affects treatment assignments and outcomes.

A less restrictive assumption is that we have access to {\it proxy variables}, which contain  relevant side information on the confounder. In the flight tickets example, we can use the number of views of the ticket reservation page as a proxy variable, which reflects peoples' desire for flights. Note that if we can completely recover the confounder from proxy variables, the ignorable treatment assignment assumption can be satisfied. Motivated by this, \citet{Lee2019Estimation} and \citet{Louizos2017Causal} aim to recover the distribution of confounders from proxy variables using  modern machine learning techniques such as generative adversarial networks \citep{Goodfellow2014GAN} or variational auto-encoders (VAE) \cite{Kingma2014VAE}. Although these methods exhibit powerful empirical performance, there is little theory that guarantees the correct recovery of the causal effects.

\citet{Kuroki2014Measurement} first considered the necessary conditions on proxy variables to provably recover the underlying causal effect via direct recovery of the hidden confounder. This work was in turn generalized by \citet{Miao2018Identifying}. In their work, it is shown that two types of proxy variables are sufficient to recover the true causal effects {\em without} explicitly recovering the confounder. One is an \textit{outcome-inducing proxy}, which {\it correlates with} confounders and {\it causes} the outcome, and the other is a \textit{treatment-inducing proxy} which {\it is caused by} confounders and {\it correlates with} the treatment. In the flight ticket example, we can use the number of views of the ticket reservation page as the  outcome-inducing proxy, and the cost of fuel as the treatment-inducing proxy. Given these proxy variables, \citet{Miao2018Identifying} show that the true causal effect can be recovered by solving a Fredholm integral equation, which is referred to as the proxy causal learning (PCL) problem. The PCL problem is known to have an interesting relation to the causal inference with \emph{multiple treatments}, which uses a subset of treatments as proxy variables in PCL \citep{pmlr-v139-wang21c,2011.04504}. 

Although the PCL problem has a solid theoretical grounding, the question of how to estimate the causal effect remains a practical challenge, in particular when we consider nonlinear causal relationships or high-dimensional treatments. In \citet{Miao2018Identifying}, the treatment and outcome are assumed to be categorical variables. In a follow-up study, \citet{Miao2018Confounding} show that we can learn a linear causal effect by a method of moments. \citet{Deaner2018Proxy} models the causal effect as a linear combination of nonlinear basis functions, which is learned by solving two stage regression. These two methods are extended by \citetMastouri, who learn the causal effect in a predefined reproducing kernel Hilbert space (RKHS). We provide an overview of the PCL problem and the two-stage regression in Section~\ref{sec:preliminary}. Although these methods enjoy desirable theoretical properties, the flexibility of the model is limited, since all existing work uses pre-specified features.

In this paper, we propose a novel {\it Deep Feature Proxy Variable (DFPV)} method, which is the first work to apply neural networks to the PCL problem. The technique we employ builds on earlier work in {\it instrumental variable (IV) regression}, which is a related causal inference setting to PCL. A range of deep learning methods has recently been introduced for IV regression \citep{Hartford2017DeepIV,Xu2021Learning,Bennett2019}. We propose to adopt the Deep Feature Instrumental Variable method \cite{Xu2021Learning}, which learns deep adaptive features within a two-stage regression framework. Details of DFPV are given in Section~\ref{sec:algorithm}. In Section~\ref{sec:experiment}, we empirically show that DFPV outperforms other PCL methods in several examples. We further apply PCL methods to the off-policy evaluation problem in a confounded
bandit setting, which aims to estimate the average reward of a new policy given data with confounding bias. We discuss the setting in Section~\ref{sec:algorithm}, and show the superiority of DFPV in experiments in  Section~\ref{sec:experiment}.



\section{Preliminaries} \label{sec:preliminary}
In this section, we introduce the proxy causal learning (PCL) problem and describe the existing two-stage regression methods to solve PCL.


\paragraph{Notation.} Throughout the paper, a capital letter (e.g. $A$) denotes a random variable, and we denote the set where a random variable takes values by the corresponding calligraphic letter (e.g. $\mathcal{A}$). 
The symbol $\prob{\cdot}$ denotes the probability distribution of a random variable (e.g. $\prob{A}$). We use a lowercase letter to denote a realization of a random variable (e.g. $a$). We denote the expectation over a random variable as $\mathbb{E}[\cdot]$ and $\|f\|_{\prob{\cdot}}$ as the $L^2$-norm of a function $f$ with respect to $\prob{\cdot}$; i.e. $\|f\|_{\prob{A}} = \sqrt{\expect[A]{f^2(A)}}$.

\subsection{Problem Setting for Proxy Causal Learning}
\begin{wrapfigure}[]{r}{0pt}
    \begin{tikzpicture}
        \node[state, fill=yellow] (x) at (0,0) {$A$};
    
        \node[state, fill=yellow] (y) [right =of x] {$Y$};
        \node[state, fill=yellow] (z) [above left =of x, xshift = 1.0cm, yshift=-0.3cm] {$Z$};
        \node[state, fill=yellow] (w) [above right =of y, xshift = -0.8cm, yshift=-0.3cm] {$W$};
        \node[state] (eps) [above right =of x, xshift = -0.5cm, yshift=0.3cm] {$U$};
        \path (x) edge (y);
        \path[bidirected] (z) edge (x);
        \path (eps) edge (y);
        \path (eps) edge (x);
        \path (eps) edge (z);
        \path[bidirected] (eps) edge (w);
        \path (w) edge (y);
    \end{tikzpicture} 
    \caption{Causal Graph.}
    \label{fig:causal-graph}
\end{wrapfigure}
We begin with a description of the PCL setting. We observe a treatment $A\in \mathcal{A}$, where  $\mathcal{A} \subset \mathbb{R}^{d_A}$, and the corresponding outcome $Y\in \mathcal{Y}$, where $\mathcal{Y} \subset \mathbb{R}$.
We assume that there exists an unobserved confounder $U \in \mathcal{U}$ that affects both $A$ and $Y$. The goal of PCL is to estimate the structural function $f_\mathrm{struct}(X)$ defined as 
\begin{align*}
    f_\mathrm{struct}(a)  = \expect[U]{\expect[Y]{Y|A=a, U}},
\end{align*}
which we assume to be continuous. This function is also known as the \textit{Average Treatment Effect (ATE)}. 
The challenge of estimating $f_\mathrm{struct}$ is that the confounder $U$ is not observable ---  we cannot estimate the structural function from observations $A$ and $Y$ alone. To deal with this, we assume access to a treatment-inducing proxy variable $Z$, and an outcome-inducing proxy variable $W$, which satisfy the following \textit{structural assumption} and \textit{completeness assumption}.
\begin{assum}[Structural Assumption \citep{Deaner2018Proxy,Mastouri2021Proximal}] \label{assum:stuctural}
We assume $Y \indepe Z | A, U$, and $W \indepe (A,Z) | U$.
\end{assum}
\begin{assum}[Completeness Assumption on Confounder \citep{Deaner2018Proxy,Mastouri2021Proximal}] \label{assum:completeness-confounder}
Let $l: \mathcal{U} \to \mathbb{R}$ be any square integrable function $\|l\|_{\prob{U}} < \infty$.
The following conditions hold for any $a\in\mathcal{A}$:
\begin{align*}
    &\expect{l(U) \mid A=a, W=w}=0\quad \forall w \in \mathcal{W} \quad\Leftrightarrow \quad l(u) = 0 \quad \prob{U}\text{-}\mathrm{a.e.}\\
    &\expect{l(U) \mid A=a, Z=z}=0\quad \forall z \in \mathcal{Z} \quad\Leftrightarrow \quad l(u) = 0 \quad \prob{U}\text{-}\mathrm{a.e.}
\end{align*}
\end{assum}
Figure~\ref{fig:causal-graph} shows the causal graph describing these relationships. 
In our setting, we assume that there is no observable confounder, though this may be easily included \citep{Mastouri2021Proximal,Tchetgen2020} as presented in  Appendix~\ref{sec:observable-confounder}.
Here, the bidirectional arrows mean that we allow both directions or even a common ancestor variable. 
Given these assumptions, it is shown that the structural function can be expressed using a solution of an integral equation.

\begin{prop}[\citet{Miao2018Identifying}] \label{prop:ident}
Let Assumptions~\ref{assum:stuctural}, \ref{assum:completeness-confounder} and Assumptions~\ref{assu:cond-exp-compactness}, \ref{assu:cond-exp-L2}, \ref{assu:Y-square-integrable} in Appendix~\ref{sec:identifiability} hold. Then there exists at least one solution to the functional equation
\begin{align}
    \expect{Y|A=a, Z=z} = \int h(a, W) \rho_W(w|A=a, Z=z) \intd w, \label{eq:h-def}
\end{align}
which holds for any $(a,z) \in \mathcal{A} \times \mathcal{Z}$. 
Here, we denote $\rho_W(w|A=a, Z=z)$ as the density function of the conditional probability $\prob{W|A=a, Z=z}$. Let $h^*$ be a solution of \eqref{eq:h-def}. The structural function $f_{\mathrm{struct}}$ is given as 
\begin{align}
    f_{\mathrm{struct}}(a) = \expect[W]{h^*(a,W)} \label{eq:struct-def}.
\end{align}
\end{prop}
For completeness, we present a proof in Appendix~\ref{sec:identifiability} (Lemma~\ref{lem:exisistence} and Corollary~\ref{cor:ate}), which is due to \citet{Miao2018Confounding} and \citet{Deaner2018Proxy}. Following \citet{Miao2018Confounding}, we call $h^*$ the \textit{bridge function}. From Proposition~\ref{prop:ident}, we can see that the estimation of the structural function reduces to the estimation of the bridge function, since once we obtain the bridge function, the structural function directly follows from \eqref{eq:struct-def}.

\subsection{Two-stage Regression Scheme for Proximal Causal Learning with Fixed Features}
To obtain the bridge function, we need to solve the functional equation \eqref{eq:h-def}. However, directly solving \eqref{eq:h-def} in a rich function space can be ill-posed (see discussion in \citet{Nashed1974Generalized}). To address this, recent works \citep{Deaner2018Proxy, Mastouri2021Proximal} minimize the following regularized loss $\mathcal{L}_{\mathrm{PV}}$ to obtain an estimate of the bridge function $\hat{h}$:
\begin{align}
    \hat{h} = \argmin_{h\in \mathcal{H}_h} \mathcal{L}_{\mathrm{PV}}(h), \quad \mathcal{L}_{\mathrm{PV}}(h) = \expect[Y,A,Z]{(Y - \expect[W|A,Z]{h(A, W)})^2} + \Omega(h), \label{eq:total-loss}
\end{align}
where $\mathcal{H}_h$ is an arbitrary space of continuous functions and $\Omega(h)$ is a regularizer on $h$. Note that this loss can be interpreted as the deviation of the r.h.s and the l.h.s in \eqref{eq:h-def} measured in $L^2$-norm with respect to the distribution $\prob{Y,A,Z}$.

\citet{Deaner2018Proxy} and \citetMastouri solve the minimization problem \eqref{eq:total-loss} by successively solving two-stage regression problems. They model the bridge function as
\begin{equation}
h(a,w) = \vec{u}^\top(\vec{\psi}_A(a) \otimes \vec{\psi}_W(w)) \label{eq:def_bridgefn}
\end{equation}
where $\vec{u}$
is a learnable weight vector,
$\vec{\psi}_A(a), \vec{\psi}_W(w)$ are vectors of fixed basis functions, and $\otimes$ is a Kronecker product, defined as $\vec{a} \otimes \vec{b} = \mathrm{vec}(\vec{a}\vec{b}^\top)$ for any finite dimensional vectors $\vec{a}, \vec{b}$.\footnote{Throughout this paper, we assume the number of basis functions to be finite.  \citetMastouri~consider an infinite number of basis function in a reproducing kernel Hilbert space, and  use the definitions of the inner and Kronecker products for that space.}
An estimate $\hat{\vec{u}}$ is obtained by solving the successive regression problems. In Stage 1, we estimate the conditional expectation $\expect[W|A=a,Z=z]{\vec{\psi}_W(W)}$ as a function of $a,z$. 
In Stage 2, we substitute the model  
\eqref{eq:def_bridgefn} into the inner conditional expectation in $\mathcal{L}_{\mathrm{PV}}$,
\[\expect[W|A=a,Z=z]{h(a,W)} = \vec{u}^\top(\vec{\psi}_A(a) \otimes \expect[W|A=a,Z=z]{\vec{\psi}_W(W)}),\]
and then minimize $\mathcal{L}_{\mathrm{PV}}$ with respect to $\vec{u}$ using the estimate of $\expect[W|A=a,Z=z]{\vec{\psi}_W(W))}$ from Stage~1. 

The above idea can be implemented as follows. We model the conditional expectation as 
\[\expect[W|A=a,Z=z]{\vec{\psi}(W)} = \vec{V} (\vec{\phi}_A(a) \otimes \vec{\phi}_Z(z)),\]
where $\vec{\phi}_A(a), \vec{\phi}_Z(z)$ are another set of basis functions, and $\vec{V}$ is a \emph{matrix} to be learned. Note that we can use different basis functions for $\vec{\phi}_A(a)$ and $\vec{\psi}_A(a)$. In Stage 1, the matrix $\vec{V}$ is learned by minimizing the following loss,
\begin{align}
    \mathcal{L}_1(\vec{V}) = \expect[W,A,Z]{\|\vec{\psi}_W(W) - \vec{V}(\vec{\phi}_A(A) \otimes\vec{\phi}_Z(Z)) \|^2} + \lambda_1 \|\vec{V}\|^2, \label{eq:stage1-loss}
\end{align}
where $\lambda_1>0$ is a regularization parameter. This is a linear ridge regression problem with multiple targets, which can be solved analytically. In Stage 2, given the minimizer $\hat{\vec{V}} = \argmin_{\vec{V}}  \mathcal{L}_1(\vec{V})$, 
we can obtain $\hat{\vec{u}}$ by minimizing the loss 
\begin{align}
    \mathcal{L}_2(\vec{u}) = \expect[Y,A,Z]{\|Y - \vec{u}^\top (\vec{\psi}_A(A) \otimes (\hat{\vec{V}} (\vec{\phi}_A(A) \otimes\vec{\phi}_Z(Z)))) \|^2} + \lambda_2 \|\vec{u}\|^2, \label{eq:stage2-loss}
\end{align}
where $\lambda_2>0$ is another regularization parameter. Stage 2 corresponds to another linear ridge regression from input $\vec{\psi}_A(A) \otimes (\hat{\vec{V}} (\vec{\phi}_A(A) \otimes\vec{\phi}_Z(Z)))$ to target $Y$, and also enjoys a closed-form solution. Given the learned weights $\hat{\vec{u}} = \argmin_{\vec{u}}  \mathcal{L}_2(\vec{u})$, the estimated structural function is $\hat{f}_{\mathrm{struct}}(a) = \hat{\vec{u}}^\top (\vec{\psi}_A(a) \otimes \expect[W]{\vec{\psi}_W(W)})$.

When fixed feature dictionaries are used, this two-stage regression benefits from strong theoretical guarantees \citep{Deaner2018Proxy,Mastouri2021Proximal}. The use of pre-specified feature maps, limits the scope and flexibility of the method, however, especially if the treatment and proxies are high dimensional (e.g. images or text), and the causal relations are nonlinear. To overcome this, we propose to use adaptive features, expressed by neural nets, as described in the next section.

\section{Deep Feature Proxy Causal Learning}\label{sec:algorithm}

In this section, we develop the DFPV algorithm, which learns adaptive features modeled by neural nets using a technique similar to \citet{Xu2021Learning}. As in \citetMastouri, 
we assume that we do not necessarily have access to observations from the joint distribution of $(A,Y,Z,W)$. Instead, we are given $m$ observations of $(A,Z,W)$ for Stage 1 and $n$ observations of $(A,Z,Y)$ for Stage 2. We denote the stage 1 observations by $(a_i, z_i, w_i)$ and the stage 2 observations by $(\tilde a_i, \tilde z_i, \tilde y_i)$. If observations of $(A,Y,Z,W)$ are given for both stages, we can evaluate the out-of-sample loss of Stage 1 using Stage 2 data and vice versa, and these losses can be used for hyper-parameter tuning of $\lambda_1,\lambda_2$ (Appendix~\ref{sec:hyper-param}). We first introduce two-stage regression with adaptive feature maps and then describe the detailed learning procedure of DFPV.

\subsection{Two-stage regression with adaptive features}

In DFPV, we consider the following models of the bridge function $h(a,w)$ and conditional feature mean $\expect[W|A,Z]{\vec{\psi}_{\theta_{W}}(W)}$:
\begin{align*}
h(a,w) = \vec{u}^\top (\vec{\psi}_{\theta_{A(2)}}(a) \otimes \vec{\psi}_{\theta_{W}}(w)), 
\quad
\expect[W|A=a,Z=z]{\vec{\psi}_{\theta_{W}}(W)} = \vec{V}(\vec{\phi}_{\theta_{A(1)}}(a) \otimes \vec{\phi}_{\theta_Z}(z)), 
\end{align*}
where $\vec{u}$ and $\vec{V}$ are parameters, and $\vec{\phi}_{\theta_{A(1)}}, \vec{\phi}_{\theta_Z},\vec{\psi}_{\theta_{A(2)}}, \vec{\psi}_{\theta_{W}}$ are neural nets parametrized by $\theta_{A(1)}, \theta_Z, \theta_{A(2)}, \theta_{W}$, respectively. Again, we may use different neural nets in the treatment features $\vec{\phi}_{\theta_{A(1)}}$ and $\vec{\psi}_{\theta_{A(2)}}$.

As in the existing work \citep{Deaner2018Proxy,Mastouri2021Proximal},
we learn $\expect[W|a,z]{\vec{\psi}_{\theta_{W}}(w)}$ in Stage 1 and $h(a,w)$ in Stage 2, but in addition to the weights $\vec{u}$ and $\vec{V}$, we also learn the parameters of the feature maps. 
Specifically, in Stage 1, we learn $\vec{V}$ and parameters $\theta_{A(1)}, \theta_Z$ by minimizing the following empirical loss:
\begin{align*}
     \hat{\mathcal{L}}_1(\vec{V}, \theta_{A(1)}, \theta_Z) = \frac1m \sum_{i=1}^m \left\|\vec{\psi}_{\theta_{W}}(w_i) - \vec{V}\left(\vec{\phi}_{\theta_{A(1)}}(a_i) \otimes\vec{\phi}_{\theta_Z}(z_i)\right) \right\|^2 + \lambda_1 \|\vec{V}\|^2.
\end{align*}
Note that $\hat{\mathcal{L}}_1$ is an empirical estimate of $\mathcal{L}_1$ in \eqref{eq:stage1-loss} with adaptive feature maps. 
Although $\hat{\mathcal{L}}_1$ depends on $\theta_{W}$, at this stage, we do not optimize $\theta_{W}$ with $\hat{\mathcal{L}}_1$ as ${\vec{\psi}_{\theta_{W}}(w)}$ is the  ``target variable'' in Stage~1. Given the minimizers $(\hat{\vec{V}}, \hat{\theta}_{A(1)}, \hat{\theta}_{Z}) = \argmin \hat{\mathcal{L}}_1$, we learn weights $\vec{u}$ and parameters $\theta_W, \theta_{A(2)}$ by minimizing the empirical stage 2 loss,
\begin{align*}
     \hat{\mathcal{L}}_2(\vec{u}, \theta_W, \theta_{A(2)}) = \frac1n \sum_{i=1}^n \left(\tilde y_i - \vec{u}^\top \left(\vec{\psi}_{\theta_{A(2)}}(\tilde a_i) \otimes \hat{\vec{V}}  \left(\vec{\phi}_{\hat{\theta}_{A(1)}}(\tilde a_i) \otimes\vec{\phi}_{\hat{\theta}_Z}(\tilde z_i)\right)\right)\right)^2 + \lambda_2 \|\vec{u}\|^2.
\end{align*}
Again, the loss $\hat{\mathcal{L}}_2$ is an empirical estimate of $\mathcal{L}_2$ in \eqref{eq:stage2-loss} with adaptive feature maps. 
Although the expression of $\hat{\mathcal{L}}_2$ does not explicitly contain $\theta_W$, 
it implicitly depends on $\theta_W$ as $(\hat{\vec{V}}, \hat{\theta}_{A(1)}, \hat{\theta}_{Z})$ is the solution of a minimization problem involving $\theta_W$. This implicit dependency makes it challenging to update $\theta_W$, as we cannot directly obtain its gradient. One possible solution is to use the implicit gradient method \citep{Lorraine2020Optimizing}, but this approach might be too computationally expensive. Instead, we use the method proposed in \citet{Xu2021Learning}, in which we ignore the dependency of $\theta_W$ on parameters $\hat{\theta}_{A(1)}, \hat{\theta}_{Z}$, and compute the gradient  via the closed-form solution of $\hat{V}$.

\subsection{Deep Feature Proxy Variable Method}

We now describe the learning procedure for DFPV. 
First, we fix parameters in the adaptive feature maps $(\theta_{A(1)}, \theta_Z, \theta_{A(2)}, \theta_{
W})$. Given these parameters, optimal weights $\hat{\vec{V}}, \hat{\vec{u}}$ can be learned by minimizing the empirical stage 1 loss $\hat{\mathcal{L}}_1$ and empirical stage 2 loss $\hat{\mathcal{L}}_2$, respectively. These minimizations can be solved analytically, where the solutions are
\begin{align}
    &\hat{\vec{V}}(\vec{\theta}) = \Psi_1^\top \Phi_1 (\Phi_1^\top \Phi_1 + m\lambda_1 I)^{-1}, &
    &\hat{\vec{u}}(\vec{\theta}) = \left(\Phi_2 ^\top \Phi_2 + n\lambda_2 I\right)^{-1}\Phi_2^\top \vec{y}_2, \label{eq:weights-sol}
\end{align}
where we denote $\vec{\theta} = (\theta_{A(1)}, \theta_Z, \theta_{A(2)}, \theta_{
W})$ and define matrices as follows:
\begin{align*}
    &\Psi_1 = \left[\vec{\psi}_{\theta_W}(w_1), \dots, \vec{\psi}_{\theta_W}(w_m)\right]^\top, & &\Phi_1 = [\vec{v}_1(a_1, z_1), \dots, \vec{v}_1(a_m, z_m)]^\top,\\
    &\vec{y}_2 = [\tilde y_1, \dots, \tilde y_n]^\top, & &\Phi_2 = [\vec{v}_2(\tilde a_1, \tilde z_1), \dots, \vec{v}_2(\tilde a_n, \tilde z_n)]^\top,&\\
    &\vec{v}_1(a, z) = \vec{\phi}_{\theta_{A(1)}}(a) \otimes\vec{\phi}_{\theta_Z}(z), &&\vec{v}_2(a, z) = \vec{\psi}_{\theta_{A(2)}}(a) \otimes \left(\hat{\vec{V}}(\vec{\theta}) \left(\vec{\phi}_{\theta_{A(1)}}(a) \otimes\vec{\phi}_{\theta_Z}(z)\right)\right).
\end{align*}
Given these weights $\hat{\vec{u}}(\vec{\theta}), \hat{\vec{V}}(\vec{\theta})$, we can update feature parameters by a gradient descent method with respect to the residuals of the loss of each stage, while regrading $\hat{\vec{V}}$ and $\hat{\vec{u}}$ as functions of parameters. Specifically, we take the gradient of the losses
\begin{align*}
    &\hat{\mathcal{L}}^{\mathrm{DFPV}}_1(\vec{\theta}) = \frac1m \sum_{i=1}^m \left\|\vec{\psi}_{\theta_{W}}(w_i) - \hat{\vec{V}}(\vec{\theta})\left(\vec{\phi}_{\theta_{A(1)}}(a_i) \otimes\vec{\phi}_{\theta_Z}(z_i)\right) \right\|^2 + \lambda_1 \|\hat{\vec{V}}(\vec{\theta})\|^2,\\
    &\hat{\mathcal{L}}^{\mathrm{DFPV}}_2(\vec{\theta}) = \frac1n \sum_{i=1}^n \left(\tilde y_i - \hat{\vec{u}}(\vec{\theta})^\top \left(\vec{\psi}_{\theta_{A(2)}}(\tilde a_i) \otimes \hat{\vec{V}}(\vec{\theta})  \left(\vec{\phi}_{\theta_{A(1)}}(\tilde a_i) \otimes\vec{\phi}_{\theta_Z}(\tilde z_i)\right)\right)\right)^2 \!+\!  \lambda_2 \|\hat{\vec{u}}(\vec{\theta})\|^2,
\end{align*}
where $\hat{\vec{V}}(\vec{\theta}), \hat{\vec{u}}(\vec{\theta})$ are given in \eqref{eq:weights-sol}. Given these losses, $(\theta_{A(1)}, \theta_{Z})$ are minimized with respect to $\hat{\mathcal{L}}^{\mathrm{DFPV}}_1(\vec{\theta})$, and $(\theta_{A(2)}, \theta_{W})$ are minimized with respect to $\hat{\mathcal{L}}^{\mathrm{DFPV}}_2(\vec{\theta})$. Finally, we take the empirical mean of $\vec{\psi}_{\theta^{(t)}_W}$ based on additional output-proxy data $S_W = \{w^{\mathrm{extra}}_i\}_{i=1}^{n_W}$, which is used for estimating the structural function. Here, we assume access to $S_W$ for proving consistency results, but empirically, we can use stage 1 data to compute this mean.

The complete procedure is presented in Algorithm~\ref{algo}. Note that we may use any sophisticated gradient-based learning method to optimize, such as Adam \citep{Kingma2015Adam}. As reported in \citet{Xu2021Learning}, we observe that  the learning procedure is stabilized by running  several gradient descent steps on the stage 1 parameters $(\theta_{A(1)}, \theta_{Z})$ before updating the stage 2 features $(\theta_{A(2)}, \theta_{W})$. Furthermore, we may use mini-batch updates, which sample subsets of the data at the beginning of each iteration and only use these subsamples to update the parameters. 

 \begin{algorithm}
 \caption{Deep Feature Proxy Causal Learning}
 \begin{algorithmic}[1]  \label{algo}
 \renewcommand{\algorithmicrequire}{\textbf{Input:}}
 \renewcommand{\algorithmicensure}{\textbf{Output:}}
 \REQUIRE Stage 1 data $S_1 =\{a_i, z_i, w_i\}$, Stage 2 data $S_2 = \{\tilde a_i \tilde z_i, \tilde y_i\}$, Additional outcome-proxy data $S_W =\{w^{\mathrm{extra}}_i\}$, Regularization parameters $(\lambda_1, \lambda_2)$. Initial values $\vec{\theta}^{(0)} = (\theta^{(0)}_{A(1)}, \theta^{(0)}_Z, \theta^{(0)}_{A(2)}, \theta^{(0)}_{W})$. Learning rate $\alpha$.
 \ENSURE Estimated structural function $\hat{f}_\mathrm{struct}(a)$
 \STATE $t \leftarrow 0$
 \REPEAT
    \STATE Compute $\hat{\vec{V}}(\vec{\theta}^{(t)}), \hat{\vec{u}}(\vec{\theta}^{(t)})$ in \eqref{eq:weights-sol} 
    \STATE Update parameters in features $\vec{\theta}^{(t+1)} \leftarrow (\theta^{(t+1)}_{A(1)}, \theta^{(t+1)}_Z, \theta^{(t+1)}_{A(2)}, \theta^{(t+1)}_{W})$ as follows
    \abovedisplayskip     =0.5\abovedisplayskip
  \abovedisplayshortskip=0.5\abovedisplayshortskip
  \belowdisplayskip     =0.5\belowdisplayskip
  \belowdisplayshortskip=0.5\belowdisplayshortskip
    \begin{align*}
        &\theta^{(t+1)}_{A(1)} \leftarrow \theta^{(t)}_{A(1)} - \alpha \nabla_{\theta_{A(1)}}\hat{\mathcal{L}}^{\mathrm{DFPV}}_1(\vec{\theta})|_{\vec{\theta} =\vec{\theta}^{(t)}},& & \theta^{(t+1)}_{Z} \leftarrow \theta^{(t)}_{Z} - \alpha \nabla_{\theta_{Z}}\hat{\mathcal{L}}^{\mathrm{DFPV}}_1(\vec{\theta})|_{\vec{\theta} =\vec{\theta}^{(t)}}\\
        &\theta^{(t+1)}_{A(2)} \leftarrow \theta^{(t)}_{A(2)} - \alpha \nabla_{\theta_{A(2)}}\hat{\mathcal{L}}^{\mathrm{DFPV}}_2(\vec{\theta})|_{\vec{\theta} =\vec{\theta}^{(t)}}, & & \theta^{(t+1)}_{W} \leftarrow \theta^{(t)}_{W} - \alpha \nabla_{\theta_{W}}\hat{\mathcal{L}}^{\mathrm{DFPV}}_2(\vec{\theta})|_{\vec{\theta} =\vec{\theta}^{(t)}}
    \end{align*}
    \STATE Increment counter $t \leftarrow t+1;$
  \UNTIL{\textbf{convergence}}
  \STATE Compute $\hat{\vec{u}}(\vec{\theta}^{(t)})$ from \eqref{eq:weights-sol} 
  \STATE Compute mean feature for $W$ using stage 1 dataset: $\vec{\mu}_{\theta_W} \leftarrow \frac1n \sum \vec{\psi}_{\theta^{(t)}_W}(w^{\mathrm{extra}}_i)$
 \RETURN $\hat{f}_\mathrm{struct}(a) = (\hat{\vec{u}}^{(t)})^\top \left(\vec{\psi}_{\hat{\theta}^{(t)}_{A(2)}}(a) \otimes \vec{\mu}_{\theta_W}\right)$
 \end{algorithmic} 
 \end{algorithm}
 
\subsection{Consistency of DFPV}
In this section, we show that the DFPV method yields a consistent estimate of the bridge function.
Our approach differs from the consistency result of \citet{Deaner2018Proxy} and \citetMastouri, which depend on the assumption that the true bridge function lies in a specific functional space (such as an RKHS \citep{Mastouri2021Proximal}) or satisfies certain smoothness properties \citep{Deaner2018Proxy}.
Since our features are adaptive, however, we instead build our result on a Rademacher complexity argument \citep{FoundationOfML}, which holds for a wide variety of hypothesis spaces. Furthermore, we do not assume that our hypothesis space contains the true bridge function, which makes this theory of independent interest.

To derive the excess risk bound, we need another completeness assumption.

\begin{restatable}[Completeness Assumption on Outcome-Inducing Proxy \citep{Deaner2018Proxy,Mastouri2021Proximal}]{assum}{completenessproxy} \label{assum:completeness-proxy}
Let $l: \mathcal{W} \to \mathbb{R}$ be any square integrable function $\|l\|_{\prob{W}} < \infty$, We assume the following condition:
\begin{align*}
    &\expect{l(W) \mid A=a, Z=z}=0\quad \forall (a,z) \in \mathcal{A} \times \mathcal{Z} \quad\Leftrightarrow \quad l(w) = 0 \quad \prob{W}\text{-}\mathrm{a.e.}
\end{align*}
\end{restatable}
This assumption is not necessary for identification (Proposition~\ref{prop:ident}), but we need it for connecting the two-stage loss and the deviation of the bridge function. Given this assumption, we have the following consistency result.

\begin{prop} \label{prop:consistency-structural}
     Let Assumption~\ref{assum:completeness-proxy} and Assumptions~\ref{assume-consist} and \ref{assume-good-support} in Appendix~\ref{sec:consistency} hold. Given stage 1 data $S_1 = \{(w_i, a_i, z_i)\}_{i=1}^m$, stage 2 data $S_2 = \{(\tilde y_i, \tilde a_i, \tilde z_i)\}_{i=1}^n$, and additional output-proxy data $S_W = \{w_i^{\mathrm{extra}}\}_{i=1}^{n_W}$, then for the minimizer of $\hat{\mathcal{L}}_2$ denoted as $(\hat{\vec{u}}, \hat{\theta}_{A(2)}, \hat{\theta}_{Z})$, we have
    \begin{align*}
        &\|f_{\mathrm{struct}} - \hat{f}_{\mathrm{struct}} \|_{\prob{A}} \leq O_p\left( \sqrt{\kappa_1 + \hat{\mathfrak{R}}_{S_1}(\mathcal{H}_1)+ \frac{1}{\sqrt{m}}} + \kappa_2\!+\!\sqrt{\hat{\mathfrak{R}}_{S_2}(\mathcal{H}_2)+  \frac{1}{\sqrt{n}}}  + \frac{1}{\sqrt{n_W}}\right),
    \end{align*}
    where $\mathcal{H}_1, \mathcal{H}_2$ are the functional classes that are considered in each stage, comprising both the predictions and their associated losses, which are defined in \eqref{eq:def-h1} and \eqref{eq:def-h2}, respectively; and $\hat{\mathfrak{R}}_{S_1}(\mathcal{H}_1), \hat{\mathfrak{R}}_{S_2}(\mathcal{H}_2)$ are their respective empirical Rademacher complexities; the quantities $\kappa_1, \kappa_2$ measure the misspecification in each regression, and are defined in \eqref{eq:def-kappa1} and \eqref{eq:def-kappa2}, respectively.
\end{prop}

We formally restate Proposition~\ref{prop:consistency-structural} in Appendix~\ref{sec:consistency} (Theorem~\ref{thm:consistency-structural}), in which the definitions of hypothesis spaces and constants can be found. From this, we can see that if we correctly identify the hypothesis space ($\kappa_1, \kappa_2 = 0$) and have vanishing Rademacher complexities ($\hat{\mathfrak{R}}_{S_1}(\mathcal{H}_1), \hat{\mathfrak{R}}_{S_2}(\mathcal{H}_2) \to 0$ in probability), then the estimated bridge function converges in probability to the true one.
Note that the Rademacher complexity of certain neural networks is known; e.g., ReLU two-layer networks are known to have the complexity of order $O(\sqrt{1/n})$ \citep{neyshabur2018the}. However, it might be not straightforward to derive the Rademacher complexities of the specific classes of neural networks we use, because we are employing the outer product of two neural networks in \eqref{eq:stage1-loss} and \eqref{eq:stage2-loss}, whose Rademacher complexity remains a challenging open problem. The analysis in Proposition~\ref{prop:consistency-structural} is similar to \citet{Xu2021Learning}, 
with two important distinctions: first, the target $f_\mathrm{struct}$ is estimated using the bridge function $h$ marginalized over one of its arguments $W$, but PCL learns the bridge function itself. Second, our result is more general and includes the misspecified case, which was not treated in the earlier work.

One limitation of our analysis is that we leave aside questions of optimization. As we previously discussed, $\hat{\mathcal{L}}_2$ does not explicitly include $\theta_W$, thus it is difficult to develop a consistency result that includes the optimization. We emphasize that Algorithm \ref{algo} does not guarantee to converge to the global minimizer, since we ignore the implicit dependency of $\theta_W$ on $\theta_{A(1)}, \theta_{Z}$ when calculating the gradient. As we will see in the Section~\ref{sec:experiment}, however, the proposed method outperforms competing methods with fixed feature dictionaries, which do not have this issue.

\subsection{DFPV for Policy Evaluation}
PCL methods can estimate not only the structural function, but also other
causal quantities, such as Conditional Average Treatment Effect (CATE) \citep{Singh2020} and
Average Treatment Effect on Treated (ATT) \citep{Deaner2018Proxy,Singh2020}. Beyond the causal context, \citep{Tennenholtz2020OffPolicy} uses a PCL method to conduct the off-policy evaluation in a partially observable Markov decision process with discrete action-state spaces. Here, we present the application of PCL methods to bandit off-policy evaluation with continuous actions, and propose a method to leverage adaptive feature maps in this context.

In bandit off-policy evaluation, we regard the outcome $Y$ as the reward and aim to evaluate the value function of a given policy. Denote policy $\pi:\mathcal{C} \to \mathcal{A}$ and define the value function $v(\pi)$ as
\begin{align*}
    v(\pi) = \expect[U,C]{\expect{Y|A=\pi(C), U}},
\end{align*}
where $C \in \mathcal{C}$ is a variable on which the policy depends. In the flight ticket sales prediction example, a company might be interested in predicting the effect of discount offers. This requires us to solve the policy evaluation with $C=A$ since a new price is determined based on the current price. Alternatively, if the company is planning to introduce a new price policy that depends on the fuel cost, we need to consider the policy evaluation with $C=Z$, since the fuel cost can be regarded as the treatment-inducing proxy variable. Here, we assume the policy $\pi$ to be deterministic, but this can be easily generalized to a stochastic policy. Although the value function $v$ contains the expectation over the latent confounder $U$, it can be estimated through a bridge function.

\begin{prop}\label{prop:value-function}
Assume that the true bridge function $h^*(a,w):\mathcal{A}\times\mathcal{W}\to \mathbb{R}$ is jointly measurable. 
Suppose $C \indepe W | U$. 
Under Assumptions~\ref{assum:stuctural}, \ref{assum:completeness-confounder} and Assumptions~\ref{assu:cond-exp-compactness}, \ref{assu:cond-exp-L2}, \ref{assu:Y-square-integrable} in Appendix~\ref{sec:identifiability}, 
we have 
\begin{align*}
    v(\pi) = \expect[C,W]{h^*(\pi(C), W)}.
\end{align*}
\end{prop}
The proof can be found in Appendix~\ref{sec:identifiability} (Corollary~\ref{cor:value-func}). From Proposition~\ref{prop:value-function}, we can obtain the value function by taking the empirical average over the bridge function $h(\pi(C), W)$. In DFPV, we estimate the bridge function as $h(a,w) = (\hat{\vec{u}})^\top (\vec{\psi}_{\hat{\theta}_{A(2)}}(a) \otimes \vec{\psi}_{\hat{\theta}_W}(w))$ where $\hat{\vec{u}},\hat{\theta}_{A(2)}, \hat{\theta}_W$ minimize $\hat{\mathcal{L}}_2$. Hence, 
given $n'$ observation of $(C,W)$ denoted as $\{\check{c}_i, \check{w}_i\}_{i=1}^{n'}$, we can estimate the value function $v(\pi)$ as 
\begin{align*}
    \hat{v}(\pi) = \frac1{n_C} \sum_{i=1}^{n'} \hat{\vec{u}}^\top \left(\vec{\psi}_{\hat{\theta}_{A(2)}}(\pi(\check{c}_i)) \otimes \vec{\psi}_{\hat{\theta}_W}(\check{w}_i))\right).
\end{align*}
We derive the following consistency result for policy value evaluation in a bandit with confounding bias.
\begin{prop} \label{prop:consistency-value}
     Let Assumption~\ref{assum:completeness-proxy} and Assumptions~\ref{assume-consist},~\ref{assume-good-support},~\ref{assume-policy-l2} in Appendix~\ref{sec:consistency} hold. Denote the minimizers of $\hat{\mathcal{L}}_2$  as $(\hat{u}, \hat{\theta}_{A(2)}, \hat{\theta}_{Z})$.  Given stage 1 data $S_1 = \{(w_i, a_i, z_i)\}_{i=1}^m$, stage 2 data $S_2 = \{(\tilde y_i, \tilde a_i, \tilde z_i)\}_{i=1}^n$, and data for policy evaluation $S_3 = \{(\check{c}_i, \check{w}_i)\}_{i=1}^{n'}$,   we have
    \begin{align*}
        &\!|v(\pi) - \hat{v}(\pi)|\leq O_p\left(\sqrt{\kappa_1 + \hat{\mathfrak{R}}_{S_1}(\mathcal{H}_1)+ \frac{1}{\sqrt{m}}} + \kappa_2\!+\!\sqrt{\hat{\mathfrak{R}}_{S_2}(\mathcal{H}_2) +  \frac{1}{\sqrt{n}}} + \frac{1}{\sqrt{n'}} \right)
    \end{align*}
    where $\mathcal{H}_1, \mathcal{H}_2$ are the function classes that are considered in each stage, comprising both the predictions and their associated losses, which are defined in \eqref{eq:def-h1} and \eqref{eq:def-h2}, respectively; and their empirical Rademacher complexities are given by $\hat{\mathfrak{R}}_{S_1}(\mathcal{H}_1), \hat{\mathfrak{R}}_{S_2}(\mathcal{H}_2)$; 
    the quantities $\kappa_1, \kappa_2$ measure the misspecification in each regression, which are defined in \eqref{eq:def-kappa1} and \eqref{eq:def-kappa2}, respectively.
\end{prop}
The formal statement of Proposition~\ref{prop:consistency-value} can be found in Appendix~\ref{sec:consistency} (Theorem~\ref{thm:consistency-structural}). As in the structural function case, we can see that estimated value function converges in probability to the true one, provided that we have vanishing Rademacher complexities and there is no misspecification.

\section{Experiments} \label{sec:experiment}

In this section, we report the empirical performance of the DFPV method. First, we present the results of estimating structural functions; we design two experimental settings for low-dimensional treatments and high-dimensional treatments, respectively. Then, we show the result of applying PCL methods to the bandit off-policy evaluation problem with confounding. We include the results for problems considered in prior work in Appendix~\ref{sec:additional-experiments}. The experiments are implemented using PyTorch \citep{Pytorch}. The code is included in the supplemental material. All experiments can be run in a few minutes on Intel(R) Xeon(R) CPU E5-2698 v4 \@ 2.20GHz. 

\paragraph{Experiments for Structural Function}
We present two structural function estimation experiments. One is a demand design experiment based on a synthetic dataset introduced by \citet{Hartford2017DeepIV}, which is a standard benchmark for the instrumental variable regression. Here, we modify the data generating process to provide a benchmark for PCL methods. We consider the problem of predicting sales $Y$ from ticket price $P$, where these are confounded by a potential demand $D \in [0,10]$. To correct this confounding bias, we use the fuel price $(C_1, C_2)$ as the treatment-inducing proxy, which has an impact on price $P$, and the number of views of the ticket reservation page $V$ as the outcome-inducing proxy. Details of the data generation process can be found in Appendix~\ref{sec:demand-design-data-generation}. 

Our second structural function estimation experiment considers high-dimensional treatment variables. We test this using the dSprite dataset \citep{dsprites17}, which is an image dataset described by five latent parameters ({\tt shape,  scale, rotation, posX} and {\tt posY}). The images are $64 \times 64 = 4096$-dimensional. Based on this, \citet{Xu2021Learning} introduced the causal experiment, where the treatment is each figure, and the confounder is \texttt{posY}. Inspired by this, we consider the PCL setting that learns the same structural functions with nonlinear confounding, which is not possible to handle in the instrumental variable setting. Specifically, the structural function $f_{\mathrm{struct}}$ and outcome $Y$ are defined as
\begin{align*}
    f_{\mathrm{struct}}(a) = \frac{(\mathrm{vec}(\vec{B})^\top a)^2 - 3000}{500}, \quad Y =  12(\mathtt{posY}-0.5)^2f_{\mathrm{struct}}(A) + \varepsilon, \quad \varepsilon \sim \mathcal{N}(0, 0.5),
\end{align*}
where each element of the matrix $\vec{B} \in \mathbb{R}^{64\times 64}$ is given as $B_{ij} = |32-j| / 32$.\footnote{This differs from the experimental setting used the preceding version of this work. The earlier experiment is described in Appendix~\ref{sec:original-dsprite}, and limitations of that setting are discussed.} We fixed the {\tt shape} parameter to {\tt heart} and used other parameters as the treatment-inducing proxy $Z$. We sampled another image that shared the same \texttt{posY} as treatment $A$, which is used as output-inducing proxy $W$. Details of data generation process can be found in Appendix~\ref{sec:dsprite-design-data-generation}. 

\begin{wrapfigure}{r}{0pt}
    \begin{tikzpicture}
        \node[state, fill=yellow] (x) at (0,0) {$A$};
    
        \node[state, fill=yellow] (y) [right =of x] {$Y$};
        \node[state, fill=yellow] (w) [above right =of y, xshift = -0.8cm, yshift=-0.3cm] {$Q$};
        \node[state] (eps) [above right =of x, xshift = -0.5cm, yshift=0.3cm] {$U$};
        \path (x) edge (y);
        \path (eps) edge (y);
        \path (eps) edge (x);
        \path[bidirected] (eps) edge (w);
    \end{tikzpicture} 
    \caption{Causal Graph in CEVAE}
    \label{fig:causal-graph-cevae}
\end{wrapfigure}

We compare the DFPV method to three competing methods, namely KPV \citep{Mastouri2021Proximal}, PMMR \citep{Mastouri2021Proximal}, and an autoencoder approach derived from the CEVAE method \citep{Louizos2017Causal}. In KPV, the bridge function is estimated through the two-stage regression as described in Section \ref{sec:preliminary}, where feature functions are fixed via their kernel functions. PMMR also models the bridge function using kernel functions, but parameters are learned by moment matching. CEVAE is not a PCL method, however, it represents a state-of-the-art approach in correcting for hidden confounders using observed proxies. The causal graph for CEVAE is shown in Figure~\ref{fig:causal-graph-cevae}, and CEVAE uses a VAE \cite{Kingma2014VAE} to recover the distribution of confounder $U$ from the ``proxy'' $Q$. We make two modifications to CEVAE to apply it in our setting. First, we include both the treatment-inducing proxy $Z$ and output-inducing proxy $W$ as $Q$ in CEVAE (we emphasize that this does not follow the causal graph in Figure~\ref{fig:causal-graph-cevae}, since there exist arrows from $Q$ to $A,Y$). 
Second, CEVAE is originally used in the setting where $Q$ is conditioned on a particular value, whereas we marginalize $Q$.  
See Appendix~\ref{sec:hyper-param-and-architectures} for the choice of the network structure and hyper-parameters. 
We tuned the regularizers $\lambda_1,\lambda_2$ as discussed in Appendix~\ref{sec:hyper-param}, with the data evenly split for Stage 1 and Stage 2.  We varied the dataset size and ran 20 simulations for each setting. Results are summarized in Figure~\ref{fig:ate}.

\begin{figure}[t]
    \centering
    \begin{minipage}{0.49\hsize}
    \centering
    \includegraphics[height=150pt]{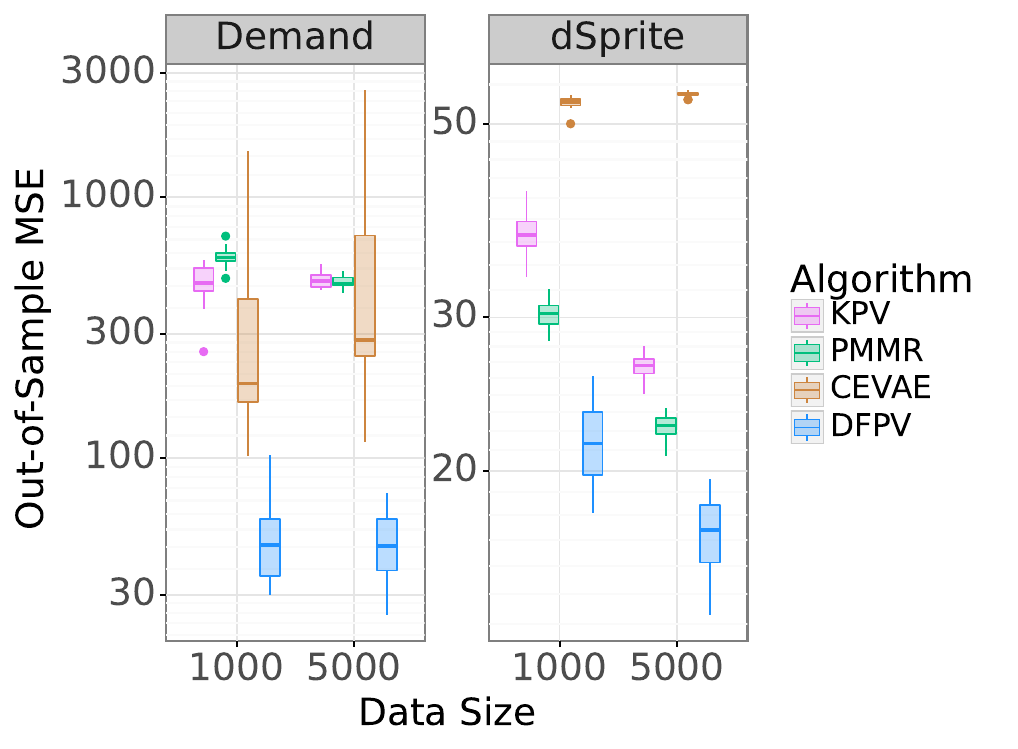}
    \caption{Result of structural function experiment in demand design setting (Left) and dSprite setting (Right).}
    \label{fig:ate}
    \end{minipage}
    \hfill
    \begin{minipage}{0.49\hsize}
        \centering
        \includegraphics[height=150pt]{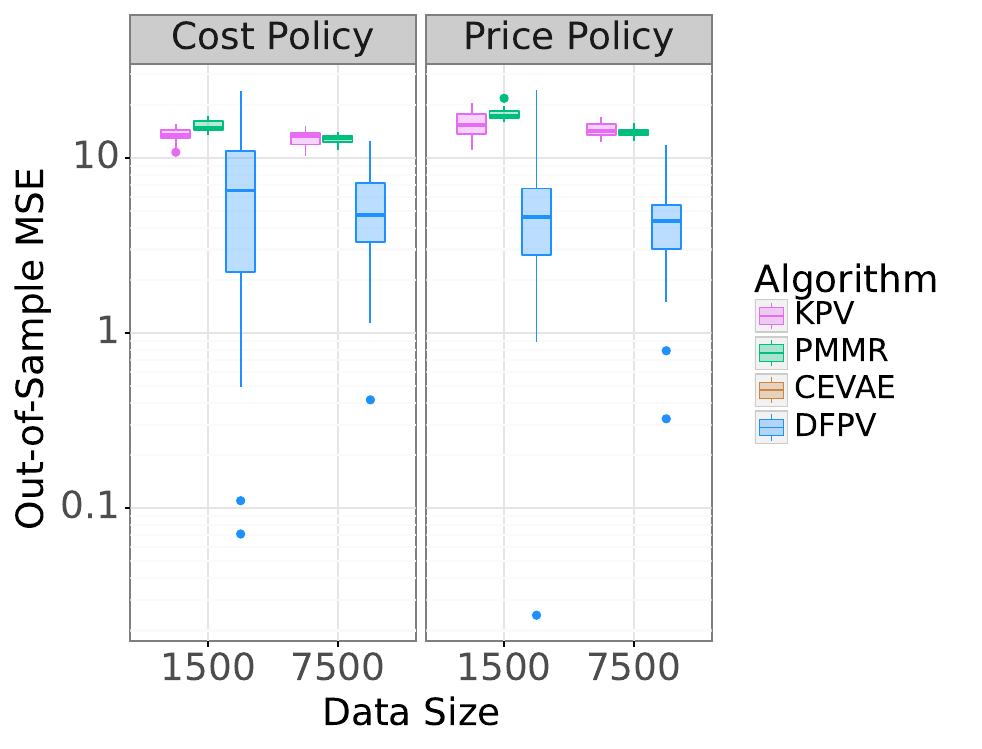}
        \caption{Result of OPE experiment when the policy depends on the costs (Left) and on the current price (Right).}
        \label{fig:ope}
        \end{minipage}
\end{figure}

In both experiments, DFPV consistently outperforms existing methods. This suggests that DFPV is capable of learning complex structural functions by taking the advantage of the flexibility of neural networks. KPV and PMMR perform similarly in Demand settings, but PMMR performs slightly better in the dSprite experiment.
Although CEVAE also learns a flexible model with a neural network, it's highly unstable in the demand design experiment and underperforms kernel methods in the dSprite experiment. 
This is because CEVAE does not take advantage of the relations between proxy variables and the structural function.

\paragraph{Experiments for Offline Policy Evaluation}
We now describe the offline policy evaluation experiment based on the demand design data. We set up synthetic experiments to evaluate two types of policy. In the first case, a policy depends on costs $C=(C_1, C_2)$, to address the question \textit{How much would we gain if we decide the price based on the fuel costs}. In the second case, the policy depends on current action $C=P$, which addresses \textit{How much would we gain if we cut the price by 30 percent}. The detailed policies can be found in Appendix~\ref{sec:ope-data-generation}. 

 We only compare PCL methods here, as it is not straightforward to apply CEVAE to the off-policy evaluation problem. We evenly split the data for Stage 1, Stage 2, and policy evaluation (i.e we set $n=m=n'$).We ran 20 simulations for each setting. Results are summarized in Figure~\ref{fig:ope}, in which DFPV performs better than existing PCL methods. This is not surprising since, as demonstrated in the structural function experiments, DFPV can estimate complex bridge functions, which results in a more accurate estimation of the value function.

\section{Conclusion}

We have proposed a novel approach for proxy causal learning, the Deep Feature Proxy Variable (DFPV) method, which performs two-stage least squares regression on flexible and expressive features.
Motivated by the literature on the instrumental variable problem, we showed how to learn these feature maps adaptively with deep neural networks. We also showed that PCL learning can be used for off-policy evaluation in the bandit setting with confounding, and that DFIV performs competitively in this domain. This work thus brings together research from the worlds of deep offline RL and causality from observational data.

In future work, it would be interesting to explore different approaches to learning deep models in the PCL problem. 
One direction would be to adapt the method of moment matching, which has been studied extensively in the instrumental variable literature \citep{Bennett2019}. Moreover, \citet{Luofeng2020Provably} recently developed a novel adversarial method to solve a class of functional equations, which would be also a promising approach in the PCL problem.  In the RL context, \citet{Tennenholtz2020OffPolicy} shows an interesting connection between PCL learning and off-policy evaluation in a partially observable Markov decision process, and we believe that adapting DFPV to this setting will be of value.

 {\bf Acknowledgments:}  We thank Olawale  Salaudeen for alerting us to an issue in the experiments of an earlier version of this document (Section \ref{sec:experiment}); and Ben Deaner, Zhu Li, and Dimitri Meunier for helpful discussions on the proxy setting. 

\bibliography{reference}

\begin{thebibliography}{36}
\providecommand{\natexlab}[1]{#1}
\providecommand{\url}[1]{\texttt{#1}}
\expandafter\ifx\csname urlstyle\endcsname\relax
  \providecommand{\doi}[1]{doi: #1}\else
  \providecommand{\doi}{doi: \begingroup \urlstyle{rm}\Url}\fi

\bibitem[Bennett et~al.(2019)Bennett, Kallus, and Schnabel]{Bennett2019}
A.~Bennett, N.~Kallus, and T.~Schnabel.
\newblock Deep generalized method of moments for instrumental variable
  analysis.
\newblock In \emph{Advances in Neural Information Processing Systems},
  volume~32, 2019.

\bibitem[Blundell et~al.(2012)Blundell, Horowitz, and Parey]{Blundell2012}
R.~Blundell, J.~Horowitz, and M.~Parey.
\newblock Measuring the price responsiveness of gasoline demand: Economic shape
  restrictions and nonparametric demand estimation.
\newblock \emph{Quantitative Economics}, 3:\penalty0 29--51, 2012.

\bibitem[Deaner(2018)]{Deaner2018Proxy}
B.~Deaner.
\newblock Proxy controls and panel data, 2018.

\bibitem[Fruehwirth et~al.(2016)Fruehwirth, Navarro, and
  Takahashi]{Fruehwirth2016How}
J.~Fruehwirth, S.~Navarro, and Y.~Takahashi.
\newblock How the timing of grade retention affects outcomes: Identification
  and estimation of time-varying treatment effects.
\newblock \emph{Journal of Labor Economics}, 34\penalty0 (4):\penalty0 979 --
  1021, 2016.

\bibitem[Goodfellow et~al.(2014)Goodfellow, Pouget-Abadie, Mirza, Xu,
  Warde-Farley, Ozair, Courville, and Bengio]{Goodfellow2014GAN}
I.~Goodfellow, J.~Pouget-Abadie, M.~Mirza, B.~Xu, D.~Warde-Farley, S.~Ozair,
  A.~Courville, and Y.~Bengio.
\newblock Generative adversarial nets.
\newblock In \emph{Advances in Neural Information Processing Systems},
  volume~27, 2014.

\bibitem[Hartford et~al.(2017)Hartford, Lewis, Leyton-Brown, and
  Taddy]{Hartford2017DeepIV}
J.~Hartford, G.~Lewis, K.~Leyton-Brown, and M.~Taddy.
\newblock Deep {IV}: A flexible approach for counterfactual prediction.
\newblock In \emph{Proceedings of The 34th International Conference on Machine
  Learning}, 2017.

\bibitem[Hill(2011)]{Hill2011Bayesian}
J.~Hill.
\newblock Bayesian nonparametric modeling for causal inference.
\newblock \emph{Journal of Computational and Graphical Statistics},
  20:\penalty0 217--240, 2011.

\bibitem[Imbens(2003)]{Imbens2004Nonparametric}
G.~W. Imbens.
\newblock Nonparametric estimation of average treatment effects under
  exogeneity: A review.
\newblock Working Paper 294, National Bureau of Economic Research, October
  2003.

\bibitem[Johansson et~al.(2016)Johansson, Shalit, and
  Sontag]{Johansson2016Learning}
F.~Johansson, U.~Shalit, and D.~Sontag.
\newblock Learning representations for counterfactual inference.
\newblock In \emph{Proceedings of The 33rd International Conference on Machine
  Learning}, 2016.

\bibitem[Kingma and Ba(2015)]{Kingma2015Adam}
D.~P. Kingma and J.~Ba.
\newblock Adam: {A} method for stochastic optimization.
\newblock In \emph{International Conference on Learning Representations}, 2015.

\bibitem[Kingma and Welling(2014)]{Kingma2014VAE}
D.~P. Kingma and M.~Welling.
\newblock Auto-encoding variational {B}ayes.
\newblock In \emph{International Conference on Learning Representations}, 2014.

\bibitem[Kress(1999)]{Kress1999linear}
R.~Kress.
\newblock \emph{Linear Integral Equations}.
\newblock Applied Mathematical Sciences. Springer New York, 1999.

\bibitem[Kuroki and Pearl(2014)]{Kuroki2014Measurement}
M.~Kuroki and J.~Pearl.
\newblock {Measurement bias and effect restoration in causal inference}.
\newblock \emph{Biometrika}, 101\penalty0 (2):\penalty0 423--437, 2014.

\bibitem[Lee et~al.(2018)Lee, Mastronarde, and van~der
  Schaar]{Lee2019Estimation}
C.~Lee, N.~Mastronarde, and M.~van~der Schaar.
\newblock Estimation of individual treatment effect in latent confounder models
  via adversarial learning.
\newblock In \emph{Advances in Neural Information Processing Systems},
  volume~32, 2018.

\bibitem[Liao et~al.(2020)Liao, Chen, Yang, Dai, Kolar, and
  Wang]{Luofeng2020Provably}
L.~Liao, Y.-L. Chen, Z.~Yang, B.~Dai, M.~Kolar, and Z.~Wang.
\newblock Provably efficient neural estimation of structural equation models:
  An adversarial approach.
\newblock In \emph{Advances in Neural Information Processing Systems},
  volume~33, 2020.

\bibitem[Lorraine et~al.(2020)Lorraine, Vicol, and
  Duvenaud]{Lorraine2020Optimizing}
J.~Lorraine, P.~Vicol, and D.~Duvenaud.
\newblock Optimizing millions of hyperparameters by implicit differentiation.
\newblock In \emph{Proceedings of the 23rd International Conference on
  Artificial Intelligence and Statistics}, 2020.

\bibitem[Louizos et~al.(2017)Louizos, Shalit, Mooij, Sontag, Zemel, and
  Welling]{Louizos2017Causal}
C.~Louizos, U.~Shalit, J.~M. Mooij, D.~Sontag, R.~S. Zemel, and M.~Welling.
\newblock Causal effect inference with deep latent-variable models.
\newblock In \emph{Advances in Neural Information Processing Systems},
  volume~31, 2017.

\bibitem[Mastouri et~al.(2021)Mastouri, Zhu, Gultchin, Korba, Silva, Kusner,
  Gretton, and Muandet]{Mastouri2021Proximal}
A.~Mastouri, Y.~Zhu, L.~Gultchin, A.~Korba, R.~Silva, M.~Kusner, A.~Gretton,
  and K.~Muandet.
\newblock Proximal causal learning with kernels: Two-stage estimation and
  moment restriction.
\newblock In \emph{Proceedings of The 38th International Conference on Machine
  Learning}, 2021.

\bibitem[Matthey et~al.(2017)Matthey, Higgins, Hassabis, and
  Lerchner]{dsprites17}
L.~Matthey, I.~Higgins, D.~Hassabis, and A.~Lerchner.
\newblock {dS}prites: Disentanglement testing sprites dataset, 2017.
\newblock URL \url{https://github.com/deepmind/dsprites-dataset/}.

\bibitem[Miao et~al.(2018{\natexlab{a}})Miao, Geng, and
  Tchetgen~Tchetgen]{Miao2018Identifying}
W.~Miao, Z.~Geng, and E.~Tchetgen~Tchetgen.
\newblock Identifying causal effects with proxy variables of an unmeasured
  confounder.
\newblock \emph{Biometrika}, 105\penalty0 (4):\penalty0 987—993,
  2018{\natexlab{a}}.

\bibitem[Miao et~al.(2018{\natexlab{b}})Miao, Shi, and
  Tchetgen]{Miao2018Confounding}
W.~Miao, X.~Shi, and E.~T. Tchetgen.
\newblock A confounding bridge approach for double negative control inference
  on causal effects, 2018{\natexlab{b}}.

\bibitem[Miao et~al.(2020)Miao, Hu, Ogburn, and Zhou]{2011.04504}
W.~Miao, W.~Hu, E.~L. Ogburn, and X.~Zhou.
\newblock Identifying effects of multiple treatments in the presence of
  unmeasured confounding, 2020.

\bibitem[Miyato et~al.(2018)Miyato, Kataoka, Koyama, and Yoshida]{Miyato2018}
T.~Miyato, T.~Kataoka, M.~Koyama, and Y.~Yoshida.
\newblock Spectral normalization for generative adversarial networks.
\newblock In \emph{International Conference on Learning Representations}, 2018.

\bibitem[Mohri et~al.(2012)Mohri, Rostamizadeh, and Talwalkar]{FoundationOfML}
M.~Mohri, A.~Rostamizadeh, and A.~Talwalkar.
\newblock \emph{Foundations of Machine Learning}.
\newblock MIT Press, 2012.

\bibitem[Nashed and Wahba(1974)]{Nashed1974Generalized}
M.~Z. Nashed and G.~Wahba.
\newblock Generalized inverses in reproducing kernel spaces: An approach to
  regularization of linear operator equations.
\newblock \emph{SIAM Journal on Mathematical Analysis}, 5\penalty0
  (6):\penalty0 974--987, 1974.

\bibitem[Neyshabur et~al.(2019)Neyshabur, Li, Bhojanapalli, LeCun, and
  Srebro]{neyshabur2018the}
B.~Neyshabur, Z.~Li, S.~Bhojanapalli, Y.~LeCun, and N.~Srebro.
\newblock The role of over-parametrization in generalization of neural
  networks.
\newblock In \emph{International Conference on Learning Representations}, 2019.
\newblock URL \url{https://openreview.net/forum?id=BygfghAcYX}.

\bibitem[Paszke et~al.(2019)Paszke, Gross, Massa, Lerer, Bradbury, Chanan,
  Killeen, Lin, Gimelshein, Antiga, Desmaison, Kopf, Yang, DeVito, Raison,
  Tejani, Chilamkurthy, Steiner, Fang, Bai, and Chintala]{Pytorch}
A.~Paszke, S.~Gross, F.~Massa, A.~Lerer, J.~Bradbury, G.~Chanan, T.~Killeen,
  Z.~Lin, N.~Gimelshein, L.~Antiga, A.~Desmaison, A.~Kopf, E.~Yang, Z.~DeVito,
  M.~Raison, A.~Tejani, S.~Chilamkurthy, B.~Steiner, L.~Fang, J.~Bai, and
  S.~Chintala.
\newblock Pytorch: An imperative style, high-performance deep learning library.
\newblock In \emph{Advances in Neural Information Processing Systems},
  volume~32, 2019.

\bibitem[Rosenbaum and Rubin(1983)]{Rosenbaum1983Propensity}
P.~R. Rosenbaum and D.~B. Rubin.
\newblock {The central role of the propensity score in observational studies
  for causal effects}.
\newblock \emph{Biometrika}, 70\penalty0 (1):\penalty0 41--55, 1983.

\bibitem[Singh(2020)]{Singh2020}
R.~Singh.
\newblock Kernel methods for unobserved confounding: Negative controls,
  proxies, and instruments, 2020.

\bibitem[Singh et~al.(2019)Singh, Sahani, and Gretton]{Singh2019Kernel}
R.~Singh, M.~Sahani, and A.~Gretton.
\newblock Kernel instrumental variable regression.
\newblock In \emph{Advances in Neural Information Processing Systems},
  volume~32, 2019.

\bibitem[Tchetgen et~al.(2020)Tchetgen, Ying, Cui, Shi, and Miao]{Tchetgen2020}
E.~J.~T. Tchetgen, A.~Ying, Y.~Cui, X.~Shi, and W.~Miao.
\newblock An introduction to proximal causal learning, 2020.

\bibitem[Tennenholtz et~al.(2020)Tennenholtz, Shalit, and
  Mannor]{Tennenholtz2020OffPolicy}
G.~Tennenholtz, U.~Shalit, and S.~Mannor.
\newblock Off-policy evaluation in partially observable environments.
\newblock In \emph{Proceedings of the AAAI Conference on Artificial
  Intelligence}, volume~34, pages 10276--10283, 2020.

\bibitem[Wang and Blei(2021)]{pmlr-v139-wang21c}
Y.~Wang and D.~Blei.
\newblock A proxy variable view of shared confounding.
\newblock In M.~Meila and T.~Zhang, editors, \emph{Proceedings of the 38th
  International Conference on Machine Learning}, volume 139 of
  \emph{Proceedings of Machine Learning Research}, pages 10697--10707. PMLR,
  18--24 Jul 2021.
\newblock URL \url{https://proceedings.mlr.press/v139/wang21c.html}.

\bibitem[Wright(1928)]{Wright1928}
P.~Wright.
\newblock \emph{The Tariff on Animal and Vegetable Oils}.
\newblock Investigations in International Commercial Policies. Macmillan
  Company, 1928.

\bibitem[Xu et~al.(2021)Xu, Chen, Srinivasan, de~Freitas, Doucet, and
  Gretton]{Xu2021Learning}
L.~Xu, Y.~Chen, S.~Srinivasan, N.~de~Freitas, A.~Doucet, and A.~Gretton.
\newblock Learning deep features in instrumental variable regression.
\newblock In \emph{International Conference on Learning Representations}, 2021.

\bibitem[Yao et~al.(2018)Yao, Li, Li, Huai, Gao, and
  Zhang]{Yao2018representation}
L.~Yao, S.~Li, Y.~Li, M.~Huai, J.~Gao, and A.~Zhang.
\newblock Representation learning for treatment effect estimation from
  observational data.
\newblock In \emph{Advances in Neural Information Processing Systems},
  volume~31, 2018.

\end{thebibliography}
\newpage
\appendix

\section{Hyper-Parameter Tuning}\label{sec:hyper-param}

If observations from the joint distribution of $(A, Y, Z, W)$ are available in both stages, we can tune the regularization parameters $\lambda_1, \lambda_2$ using the approach proposed in \citet{Singh2019Kernel, Xu2021Learning}. Let the complete data of stage 1 and stage 2 be denoted as $(a_i, y_i, z_i, w_i)$ and $(\tilde a_i, \tilde y_i, \tilde z_i, \tilde w_i)$. Then, we can use the data not used in each stage to evaluate the out-of-sample performance of the other stage. Specifically, let Algorithm~\ref{algo} converges at $t=T$, and the regularizing parameters are given by
\begin{align*}
    &\lambda_1^* = \argmin \mathcal{L}_{1\text{-oos}}, \quad \mathcal{L}_{1\text{-oos}} = \frac1n \sum_{j=1}^n \left\|\vec{\psi}_{\theta_{W}^{(T)}}(\tilde w_j) - \hat{\vec{V}}^{(T)}\left(\vec{\phi}_{\theta^{(T)}_{A(1)}}(\tilde a_i) \otimes \vec{\phi}_{\theta^{(T)}_Z}(\tilde z_i)\right)\right\|^2,\\
    &\lambda_2^* = \argmin \mathcal{L}_{2\text{-oos}}, \\
    &\quad \mathcal{L}_{2\text{-oos}} = \frac1n \sum_{i=1}^n \left(y_i - (\vec{u}^{(T)})^\top \left(\vec{\psi}_{\theta^{(T)}_{A(2)}}(a_i) \otimes \left(\hat{\vec{V}}^{(T)} \left(\vec{\phi}_{\theta^{(T)}_{A(1)}}(a_i) \otimes\vec{\phi}_{\theta^{(T)}_{Z}}(z_i)\right)\right)\right)\right)^2
\end{align*}
where $\theta_{Z}^{(T)}, \theta_{A(1)}^{(T)}, \theta_{W}^{(T)}, \theta_{A(2)}^{(T)}, \hat{\vec{V}}^{(T)}, \vec{u}^{(T)}$ are the learned parameters by Algorithm~\ref{algo}.

\section{Identifiability} \label{sec:identifiability}
In this appendix, we prove propositions given in the main text.
In the following, we assume that the spaces $\mathcal{U}$, $\mathcal{A}$, $\mathcal{Z}, \mathcal{W}$ are separable and completely metrizable topological spaces and equipped with Borel $\sigma$-algebras.
In this section, we use the notation $P_{A\given Z=z}$ to express the distribution of a random variable $A$ given another variable $Z=z$. 

\subsection{Existence of bridge function}
First, we discuss conditions to guarantee the existence of the bridge function $h$. Let us consider the following operators:
\begin{align*}
E_{a}:L^{2}(P_{W\given A=a})\to L^{2}(P_{Z|A=a}),\ E_{a}f & \coloneqq\expect{f(W)\given A=a,Z=\cdot},\\
F_{a}:L^{2}(P_{Z\given A=a})\to L^{2}(P_{W|A=a}),\ F_{a}g & \coloneqq\expect{g(Z)\given A=a,W=\cdot},
\end{align*}
where the conditional expectations are identified as equivalent classes
with natural inclusion into their corresponding $L^{2}$ spaces. Our
goal is to show that $\expect{Y\given A=a,Z=\cdot}$ is in the range of
$E_{a},$ i.e., we seek a solution of the inverse problem defined
by
\begin{equation}
E_{a} h=\expect{Y|A=a,Z=\cdot}.\label{eq:inteq}
\end{equation}
This suffices to prove the existence of the function $h,$ for if
there exists a function $h^*_{a}$ for each $a\in{\cal A}$ such that
\[
\expect{h^*_a|A=a,Z=\cdot}=\expect{Y|A=a,Z=\cdot},
\]
we can define $h^*(a,w)\coloneqq h^*_{a}(w).$ 
The existence of the bridge function corresponds to the results of \citet[][Lemma 1.1.b]{Deaner2018Proxy} and \citet[][Propisition 1]{Miao2018Identifying}.
For simplicity, in constrast to \citet[][]{Deaner2018Proxy}, we assume that there is no shared quantity between the proxy variables $W$ and $Z$. 
Following the proofs of the previous works \cite[][Appendix C, p. 57]{Deaner2018Proxy} and \cite[][Appendix 6]{Miao2018Identifying}, we aim to solve the integral equation in \eqref{eq:inteq}.

We will make use of the
following theorem on the existence of a solution of a linear integral
equation \cite[Theorem 15.18]{Kress1999linear}. 
\begin{prop}[{\cite[Theorem 15.18]{Kress1999linear}}] \label{prop:picard}Let $\mathcal{X}$ and $\mathcal{Y}$ be Hilbert spaces.
Let $E:\mathcal{X}\to \mathcal{Y}$ be a compact linear operator with singular system $\{(\mu_{n},\varphi_{n},g_{n})\}_{n=1}^{\infty}.$
The equation of the first kind 
\[
E\varphi=f
\]
 is solvable if and only if $f\in N(E^{*})^{\perp}$ and 
\[
\sum_{n=1}^{\infty}\frac{1}{\mu_{n}^{2}}\lvert\la f,g_{n}\ra\rvert^{2}<\infty.
\]
Here, $N(E^*)$ denotes the null space of the operator $E^*$. 
 Then a solution is given by 
\[
\phi=\sum_{n=1}^{\infty}\frac{1}{\mu_{n}}\la f,g_{n}\ra\varphi_{n}.
\]
\end{prop}
To apply Proposition \ref{prop:picard}, we make the following additional
assumptions.
\begin{assum}
\label{assu:cond-exp-compactness} For each $a\in{\cal A},$ the operator
$E_{a}$ is compact with singular system $\{(\mu_{a,n},\varphi_{a,n},g_{a,n})\}_{n=1}^{\infty}.$ 
\end{assum}

\begin{assum}
\label{assu:cond-exp-L2} For each $a\in{\cal A},$ the conditional
expectation $f_{Y|a}\coloneqq\expect{Y|A=a,Z=\cdot}$ satisfies
\[
\sum_{n=1}^{\infty}\frac{1}{\mu_{a,n}^{2}}\lvert\la f_{Y|a},g_{a,n}\ra_{L^{2}(P_{Z\given A=a})}\rvert^{2}<\infty,
\]
for a singular system $\{(\mu_{a,n},\phi_{a,n},g_{a,n})\}_{n=1}^{\infty}$
given in Assumption \ref{assu:cond-exp-compactness}. 
\end{assum}

\begin{rem}
Assumption~\ref{assu:cond-exp-compactness} is a minimal requisite to apply Proposition~\ref{prop:picard}. 
The existing works \cite[][]{Miao2018Identifying, Deaner2018Proxy} assume stronger conditions; the operator $E_a$ is assumed to be Hilbert-Schmidt, which implies the compactness. 
\citet[][Assumption A.1]{Deaner2018Proxy} assumes that
the joint distribution of $W$ and $Z$ conditioned on $a$ is absolutely
continuous with respect to the product measure of $P_{W|A=a}$ and
$P_{Z\given A=a},$ and its density is square integrable. The density
function serves as the integral kernel of the operator $E_{a}$ whose
$L^{2}(P_{W|A=a}\otimes P_{Z\given A=a})$-norm corresponds to the
Hilbert-Schmidt norm of the operator, where we use $\otimes$ to denote the product measure. 
On the other hand, in the setting of \citet[][]{Miao2018Identifying}, all probability distributions have densities with respect to the Lebesgue measure. The operator $E_a$ (denoted by $K_x$ in \cite[][]{Miao2018Identifying}) is defined by the relevant densities accordingly (see the paragraph after Lemma 2 of \cite[][]{Miao2018Identifying}). 
The compactness is, as in \cite{Deaner2018Proxy}, established by the square-integrability of the integral kernel, which implies that the operator $E_a$ is Hilbert-Schmidt (see Condition A1 in \cite{Miao2018Identifying}).  
\end{rem}

It is easy to see that Assumptions~\ref{assu:cond-exp-compactness}~and~\ref{assu:cond-exp-L2}
are required for using Proposition~\ref{prop:picard}. The remaining condition
to show is that $\expect{Y\given A=a,Z=\cdot}$ is in $N(E_{a}^{*})^{\perp}.$
We show that the structural assumption (Assumption~\ref{assum:stuctural}) and completeness assumption (Assumption~\ref{assum:completeness-confounder}) imply the required condition.
The result below closely follows \cite{Deaner2018Proxy} in that we share the same completeness condition. 

\begin{lem}
\label{lem:null-space-assumption} Under Assumptions \ref{assum:stuctural} and \ref{assum:completeness-confounder},
the conditional expectation $\expect{Y\given A=a,Z=\cdot}$ is in the
orthogonal complement of the null space $N(E_{a}^{*}).$
\begin{proof}
We first show that the adjoint of $E_{a}$ is given by $F_{a}.$ For
the operator $E_{a},$ any $f\in L_{2}(P_{W|A=a})$ and $g\in L_{2}(P_{Z|A=a}),$
we have 
\begin{align*}
\inner{E_{a}f}g_{L_{2}(P_{Z\given A=a})} & =\expect[Z|A=a]{\expect{f(W)\given A=a,Z}g(Z)}\\
& =\expect[Z|A=a]{\expect[U|A=a, Z]{\expect{f(W)\given A=a,Z,U}}g(Z)}\\
& \overset{\mathrm{(a)}}{=}\expect[Z|A=a]{\expect[U|A=a, Z]{\expect{f(W)\given A=a,U}}g(Z)}\\
& =\expect[U, Z|A=a]{\expect{f(W)\given A=a,U}g(Z)}\\
& =\expect[U|A=a]{\expect{f(W)\given A=a,U}\expect{g(Z)|A=a,U}}
\end{align*}
where (a) follows from $W \indepe (A,Z) | U$, which is from Assumption~\ref{assum:stuctural}. Similarly,
\begin{align*}
 & \inner f{F_{a}g}_{L^{2}(P_{W|A=a})}=\expect[W|A=a]{f(W) \expect{g(Z)\given A=a, W}}\\
 &=\expect[W|A=a]{f(W) \expect[U|A=a, W]{\expect{g(Z)\given A=a, W, U}}}\\
 &\overset{\mathrm{(b)}}{=} \expect[W|A=a]{f(W) \expect[U|A=a, W]{\expect{g(Z)\given A=a, U}}}\\
 &=\expect[W,U|A=a]{f(W)\expect{g(Z)\given A=a, U}}\\
 &=\expect[U|A=a]{\expect{f(W)\given A=a, U} \expect{g(Z)\given A=a, U}}\\
 & =\inner{E_{a}f}g_{L_{2}(P_{Z\given A=a})}.
\end{align*}
Again, (b) is given by $W \indepe A,Z | U$ from Assumption~\ref{assum:stuctural}.
For any $f^{*}\in N(E^*_{a}) = N(F_a),$ by iterated expectations, we have 
\begin{align}
0 & =\expect{f^{*}(Z)|A=a,W=\cdot}\nonumber \\
 & =\expect[U]{\expect{f^{*}(Z)\given A,U,W}\given A=a,W=\cdot}\nonumber \\
 & =\expect[U]{\expect{f^{*}(Z)\given A,U} \given A=a,W=\cdot}.\label{eq:fstar-intermediate}
\end{align}

From Assumption~\ref{assum:completeness-confounder},
\begin{align*}
    \expect{l(U) \mid A=a, W=w}=0\quad \forall (a,w) \in \mathcal{A} \times \mathcal{W} \quad\Leftrightarrow \quad l(u) = 0 \quad P_{U}\text{-}\mathrm{a.e.}
\end{align*}
for all functions $l \in L^2(P_{U|A=a})$. Note that $\expect{f^{*}(Z)\given A,U} \in  L^2(P_{U|A=a})$ since
\begin{align*}
\expect[U|A]{\EE[f^{*}(Z)\given A,U]^{2}} & \leq\expect[U|A]{\expect{f^{*}(Z)^{2}\given A,U}}\quad (\because \text{Jensen's Inequality})\\
 & =\expect{f^{*}(Z)^{2}\given A}<\infty
\end{align*}
Hence, \eqref{eq:fstar-intermediate} and Assumption~\ref{assum:completeness-confounder}
implies 
\[
\expect{f^{*}(Z)\given A=a,U=\cdot}=0 \quad P_{U}\text{-}\mathrm{a.s.}
\]
Then, the inner product between $f^{*}$ and $\expect{Y\given A=a,Z=\cdot}$
is given as follows:
\begin{align*}
\inner{f^{*}}{\expect{Y\given A=a,Z=\cdot}}_{L^{2}(P_{Z|A=a})} & =\expect[Z|A=a]{f^{*}(Z)\expect{Y\given A=a,Z}}\\
 & =\expect[Z|A=a]{f^{*}(Z)\expect[U|A=a,Z]{\expect{Y\given A=a,Z,U}}}\\
 & \overset{(c)}{=}\expect[Z|A=a]{f^{*}(Z)\expect[U|A=a,Z]{\expect{Y\given A=a,U}}}\\
 & =\expect[U, Z|A=a]{f^{*}(Z){\expect{Y\given A=a,U}}}\\
 & =\expect[U|A=a]{\expect[Z|U,A=a]{f^{*}(Z)}{\expect{Y\given A=a,U}}}\\
 & =0.
\end{align*}
Again (c) holds from $Y \indepe Z | A, U$ in Assumption~\ref{assum:stuctural}. Hence, we have 
\[
\expect{Y\given A=a,Z=\cdot}\in N(E_{a}^{*})^{\perp}.
\]
\end{proof}
\end{lem}

\begin{rem}
The difference between the first condition in Assumption \ref{assum:completeness-confounder} and
Condition 3 in \cite[][]{Miao2018Identifying} is in the approach
to establishing that the conditional expectation $\EE[Y\given A=a,Z=\cdot]$
belongs to $N(F_{a})^{\perp}.$ More specifically, Condition 3 in
\cite{Miao2018Identifying} is equivalent to having $N(F_{a})=\{0\}$ and so
any nontrivial $L^{2}(P_{Z|A=a})$-function is in the orthogonal complement.
Assumption \ref{assum:completeness-confounder} does not require $N(F_{a})=\{0\}$
but only implies 
\[
N(F_{a})\subset\{f\in L^{2}(P_{Z|A=a}):\EE[f(Z)\given A=a,U=\cdot]=0\}.
\]
For our identifiability proof, we use the weaker condition, Assumption~\ref{assum:completeness-confounder}. 
\end{rem}

Now, we are able to apply Proposition \ref{prop:picard}, leading to the following 
lemma.
\begin{lem} \label{lem:exisistence}
Under Assumptions \ref{assum:stuctural}, \ref{assum:completeness-confounder}, \ref{assu:cond-exp-compactness} and \ref{assu:cond-exp-L2}, for each $a\in{\cal A},$ there exists a function $h^*_{a}\in L_{2}(P_{W|A=a})$
such that 
\[
\expect{Y\given A=a,Z=\cdot}=\expect{h^*_{a}(W)\given A=a,Z=\cdot}.
\]
\begin{proof}
By Lemma \ref{lem:null-space-assumption}, the regression function
$\expect{Y\given A=a,Z=\cdot}$ is in $N(E_{a}^{*})^{\perp}.$ Therefore,
by Proposition \ref{prop:picard}, under the given assumptions, there exists a solution to \eqref{eq:inteq}. Letting the
solution be $h^*_{a}$ completes the proof. 
\end{proof}
\end{lem}

\subsection{Identifiability}
Here, we show that the bridge function $h$ can be used to compute the various causal quantities. In addition to the assumptions required for the existence of bridge function listed in the previous section, we assume the conditional expectation of outcome given treatment is square integrable.

\begin{assum}
\label{assu:Y-square-integrable} For each $a\in{\cal A},$ under the
observational distribution, we have 
\[
\expect{Y^{2}\given A=a}<\infty.
\]
\end{assum}
Given the assumption, we can prove the following theorem.

\begin{thm}
[Identifiability] \label{thm:proxy-identifiablity} Assume that there
exists a function $h:{\cal A}\times{\cal W}\to\mathbb{R}$ such that
for each $a\in{\cal A},$ the function $h(a,\cdot)$ is in $L^{2}(P_{W\given A=a})$
and satisfies 
\[
\expect{Y|A=a,Z=\cdot}=\expect{h(a,W)\given A=a,Z=\cdot} \quad P_{Z\given A=a}\text{-a.s.}
\]
Under Assumptions \ref{assum:stuctural}, \ref{assum:completeness-confounder} and \ref{assu:Y-square-integrable},
we have 
\begin{equation}
\expect[W|U]{h(a,W)}=\expect{Y\given A=a,U}.\label{eq:h-equal-y-given-u}
\end{equation}

\begin{proof}
Note that we have
\begin{align*}
\expect{h(a,W)\given A=a,Z} & =\expect[U|A=a, Z]{\expect{h(a,W)\given A=a,Z,U}}\\
 & =\expect[U|A=a,Z]{\expect{h(a,W)\given U}},
\end{align*}
\begin{align*}
\expect{Y\given A=a,Z} & =\expect[U|A=a, Z]{\expect{Y\given A=a,Z,U}}\\
 & =\expect[U|A=a, Z]{\expect{Y\given A=a, U}},
\end{align*}
where the second line of each equation follows from Assumption~\ref{assum:stuctural}.
Moreover, by assumption, we have 
\begin{align*}
\expect[U|A=a]{\expect{h(a,W)\given U}^{2}} & \leq \expect[W|A=a]{h(a,W)^{2}}<\infty,\ \text{and}\\
\expect[U|A=a]{\expect{Y\given A=a, U}^{2}} & \leq \expect{Y^{2} \given A=a}<\infty.
\end{align*}
Note from Assumption~\ref{assum:completeness-confounder}, we have
\[
\expect{l(U) \mid A=a, Z=z}=0\quad \forall (a,z) \in \mathcal{A} \times \mathcal{Z} \quad\Leftrightarrow \quad l(u) = 0\quad P_{U}\text{-a.e.}
\]
Therefore, by setting $l(u) = \expect{Y\given A=a, U=u} - \expect{h(a,W)\given U=u}$, for all $a \in \mathcal{A}$, we have
\begin{align*}
    &\expect{h(a,W)\given A=a,Z=\cdot} = \expect{Y\given A=a,Z=\cdot} \quad P_{Z\given A=a}\text{-a.s.}\\
    \Leftrightarrow & \expect[U|A=a,Z=\cdot]{\expect{h(a,W)\given U} - \expect{Y\given A=a,U}} = 0 \quad  P_{Z\given A=a}\text{-a.s.}\\
    \Leftrightarrow & \expect[W|U=\cdot]{h(a,W)} = \expect{Y\given A=a,U=\cdot} \quad  P_{U}\text{-a.s.}
\end{align*}
\end{proof}
\end{thm}

Using Theorem~\ref{thm:proxy-identifiablity}, we can show following two corollaries, which are used in the main body.
\begin{corollary} \label{cor:ate}
Let the assumptions in Theorem~\ref{thm:proxy-identifiablity} hold. Given a bridge function $h^*$, we can estimate structural function $f_{\mathrm{struct}}$ as 
\begin{align*}
    f_{\mathrm{struct}}(a) &\coloneqq \expect[U]{\expect{Y|A=a,U}}\\
    &= \expect[W]{h^*(a,W)}
\end{align*}
\begin{proof}
From Theorem~\ref{thm:proxy-identifiablity}, we have
\begin{align*}
    \expect[U]{\expect{Y|A=a,U}} = \expect[U]{\expect[W|U]{h^*(a,W)}} = \expect[W]{h^*(a,W)}
\end{align*}
\end{proof}
\end{corollary}

\begin{rem}
The above corollary corresponds to the identifiablity results in obtained in the previous works \cite{Miao2018Identifying, Deaner2018Proxy}. We follow the proof of Theorem 1.1.a in \cite[][Appendix B]{Deaner2018Proxy} (See also, Theorem 1 and Appendix 3 of \cite{Miao2018Identifying}).

\end{rem}

\begin{corollary} \label{cor:value-func}
Assume we are given a bridge function $h^*(a,w)$ that is jointly measurable. 
Suppose $C \indepe W |U$. 
With the assumptions in Theorem~\ref{thm:proxy-identifiablity}, 
we can write the value function of policy $\pi(C)$ as
\begin{align*}
    v(\pi) &\coloneqq \expect[C, U]{\expect{Y|A=\pi(C),U}}\\
    &= \expect[C]{\expect[W|C]{h(\pi(C),W)}}
\end{align*}
\begin{proof}
From Theorem~\ref{thm:proxy-identifiablity}, we have
\begin{align*}
    \expect[C, U]{\expect{Y|A=\pi(C),U}} = \expect[C,U]{\expect[W|U]{h^*(\pi(C),W)}} = \expect[C]{\expect[W|C]{h(\pi(C),W)}},
\end{align*}
where the last equality holds by the conditional independence $C \indepe W |U$.
\end{proof}
\end{corollary}
Note that in the existence claim, the bridge function $h^*(a, w)$ is constructed by aggregating over  $\{h^*_a\}_{a \in \mathcal{A}}$. This construction does not guarantee that the function is measurable with respect to $a$; the lack of measurability renders its expectation undefined. 
Thus, we have additionally assumed the joint measurability of the bridge function in Corollary~\ref{cor:value-func}.
Validating this assumption is crucial theoretically, and we leave it for future work.

\section{Consistency of DFPV algorithm}\label{sec:consistency}
In this appendix, we prove consistency of the DFPV approach. Following \citet{Xu2021Learning}, we establish consistency of the end-to-end procedure incorporating Stages 1 and 2. 
We establish the result by first showing a Stage 1 consistency result (Lemma~\ref{lem:1}), and then establishing the consistency of Stage 2 with the empirical Stage 1 solution used as input (Lemma~\ref{lem:2}).  
The desired result then follows in Theorem~\ref{thm:totalConsistency}. 
Here, we assume the bridge function $h^*(a,w)$ is jointly measurable so that the expectation of $h^*$ is defined. 

Consistency results will be expressed in terms of the complexity of the function classes used in Stages 1 and 2, as encoded in the Rademacher complexity of the function space of these functions (see Proposition \ref{prop:bound} below).  Consistency for particular function classes can then be shown by establishing that the respective Rademacher complexities vanish. We leave for future work the task of demonstrating this property for individual function classes of interest.




\subsection{Operator view of DFPV}

The goal of DFPV is to learn a bridge function $h^*$, which satisfies
\begin{align}
    \expect[Y|A=a, Z=z]{Y} = \expect[W|A=a, Z=z]{h^*(a,W)}. \label{eq:functional}
\end{align}
We use the model \[h(a,w) = \vec{u}^\top (\vec{\psi}_{\theta_{A(2)}}(a) \otimes \vec{\psi}_{\theta_{W}}(w))\] and denote the hypothesis spaces for $\vec{\psi}_{\theta_W}$ and $h$ as $\mathcal{H}_{\vec{\psi}_{\theta_W}}: \mathcal{W}\to \mathbb{R}^{d_W}$ and $\mathcal{H}_h: \mathcal{A}\times\mathcal{W} \to\mathbb{R}$, respectively. 
To learn the parameters, we minimize the following stage 2 loss:
\begin{align*}
    &\hat{\vec{u}}, \hat{\theta}_W, \hat{\theta}_{A(2)} = \argmin_{\vec{u}, \theta_W, \theta_{A(2)}} \hat{\mathcal{L}}_2(\vec{u},\theta_X, \theta_{A(2)}),\\
    &\hat{\mathcal{L}}_2 = \frac1n \sum_{i=1}^n (\tilde y_i - \vec{u}^\top (\vec{\psi}_{\theta_{A(2)}}(\tilde a_i)  \otimes \hatexpect[W|A,Z]{\vec{\psi}_{\theta_W}(W)}(\tilde a_i, \tilde z_i)))^2.
\end{align*}
We denote the resulting estimated bridge function as 
\[
\hat{h}(a,w) = (\hat{\vec{u}})^\top (\vec{\psi}_{\hat{\theta}_{A(2)}}(a) \otimes \vec{\psi}_{\hat{\theta}_W}(w)).
\] 
For simplicity, we set all regularization terms to zero. Here, $\hatexpect[W|A,Z]{\vec{\psi}_{\theta_W}(W)}$ is the empirical conditional expectation operator, which maps an element of $\mathcal{H}_{\vec{\psi}_{\theta_W}}$ to some function $\vec{g}(a,z)$ which is defined as
\begin{align*}
    & \hatexpect[W|A,Z]{\vec{\psi}_{\theta_W}(W)}= \argmin_{\vec{g} \in \mathcal{G}} \hat{\mathcal{L}}_1(\vec{g}; \vec{\psi}_{{\theta}_{W}}),\\
    &\hat{\mathcal{L}}_1 = \frac1n \sum_{i=1}^m \|\vec{\psi}_{{\theta}_{W}}(w_i) - \vec{g}(a_i, z_i)\|^2,
\end{align*}
where $\|\cdot\|$ is the $\ell_2$-norm, and $\mathcal{G}$ is an arbitrary function space.
In DFPV, we specify $\mathcal{G}$ as the set consisting of functions $\vec{g}$ of the form
\[
\vec{g} = {\vec{V}}(\vec{\phi}_{{\theta}_{A(1)}}(a_i) \otimes \vec{\phi}_{{\theta}_{Z}}(z_i)).
\] 
Note that this formulation is equivalent to the one introduced in Section~\ref{sec:algorithm}. With a slight abuse of notation, for $h(a,w) = \vec{u}^\top \vec{\psi}_{\theta_{A(2)}}(a) \otimes \vec{\psi}_{\theta_{W}}(w) \in \mathcal{H}_h$, we define $\hatexpect[W|A,Z]{h}$ to be
\begin{align*}
    \hatexpect[W|A,Z]{h(A,W)}(a,z) = \vec{u}^\top \left(\vec{\psi}_{\theta_{A(2)}}(a) \otimes \hatexpect[W|A,Z]{\vec{\psi}_{\theta_W}}(a,z)\right)
\end{align*}
since this is the empirical estimate of $\expect[W|A,Z]{h(A,W)}$.

\subsection{Generalization errors for regression}
Here, we bound the generalization errors of both stages using Rademacher complexity bounds \citep{FoundationOfML}.

\begin{prop}{\citep[Theorem 3.3][with slight modification]{FoundationOfML}} \label{prop:bound}
    Let $\mathcal{S}$ be a measurable space and $\mathcal{H}$ be a family of functions mapping from $\mathcal{S}$ to $[0, M]$. Given fixed dataset $S = (s_1, s_2, \dots, s_n) \in \mathcal{S}^n$, the empirical Rademacher complexity is given by
    \begin{align*}
        \hat{\mathfrak{R}}_S(\mathcal{H}) = \expect[\vec{\sigma}]{\frac1n \sup_{h\in \mathcal{H}} \sum_{i=1}^n \sigma_i h(s_i)},
    \end{align*}
    where $\vec{\sigma}=(\sigma_1,\dots,\sigma_n)$, 
    with $\sigma_i$ independent  random variables taking values in $\{-1,+1\}$ with equal probability. Then, for any $\delta > 0$, with probability at least $1-\delta$ over the draw of an i.i.d sample $S$ of size $n$,  each of following holds for all $h \in \mathcal{H}$:
    \begin{align*}
        &\expect{h(s)} \leq \frac1n \sum_{i=1}^n h(s_i) + 2 \hat{\mathfrak{R}}_S + 3M \sqrt{\frac{\log 2/\delta}{2n}},\\
        & \frac1n \sum_{i=1}^n h(s_i) \leq \expect{h(s)} + 2 \hat{\mathfrak{R}}_S + 3M \sqrt{\frac{\log 2/\delta}{2n}}.
    \end{align*}
\end{prop}

We list the assumptions below.
\begin{assum} \label{assume-consist}
   The following hold:
    \begin{enumerate}
    \item Bounded outcome variable $|Y| \leq M$.
    \item  Bounded stage 1 hypothesis space: $ \forall a \in \mathcal{A}, z \in \mathcal{Z}, \|\vec{g}(a,z)\| \leq 1$.
    \item Bounded stage 2 feature map $\forall w \in \mathcal{W},   \|\vec{\psi}_{\theta_W}(w) \| \leq 1$.
    \item  Bounded stage 2 weight: $\forall \|\vec{v}\| \leq 1, \forall a \in \mathcal{A},  |\vec{u}^\top (\vec{\psi}_{\theta_{A(2)}}(a) \otimes \vec{v})| \leq M$.
    \end{enumerate}
\end{assum}
Note that unlike \citet{Xu2021Learning}, 
we consider the case where the bridge function $h^*$ and the conditional expectation $\expect[W|A,Z]{\vec{\psi}_{\theta_W}}$ might not be included in the hypothesis spaces $\mathcal{H}_h$, $\mathcal{G}$, respectively. 

As in \citet{Xu2021Learning}, we leave aside questions of optimization. Thus, we assume that the optimization procedure over $(\theta_{A(1)}, \theta_Z,\vec{V})$ is sufficient to recover $\hatexpect[W|A,Z]{\vec{\psi}_{\theta_W}(W)},$ and that the optimization procedure over $(\theta_{A(2)}, \theta_W, \vec{u})$ is sufficient to recover $\hat{h}$ (which requires, in turn, the correct $\hat{\mathbb{E}}_{W|A,Z},$ for this $\vec{\psi}_{\theta_W}$). We emphasize that Algorithm \ref{algo} does not guarantee these properties. Based on these assumptions, we derive the generalization error in terms of $L_2$-norm with respect to $\prob{A,W}$.

Following the discussion in \citet[Lemma 1]{Xu2021Learning}, we can show the following lemma that provides a generalization error bound for Stage 1 regression.
\begin{lem}\label{lem:1}
    Under Assumption~\ref{assume-consist}, and given stage 1 data $S_1 = \{(a_i, z_i, w_i)\}_{i=1}^m$, for any $\delta>0$, with at least probability $1-2\delta$, we have
    \begin{align*}
        \left\|\hatexpect[W|A,Z]{h(A,W)} - \expect[W|A,Z]{h(A,W)}\right\|_{P(A, Z)} \leq M\sqrt{\kappa_1 + 4\hat{\mathfrak{R}}_{S_1}(\mathcal{H}_1) + 24\sqrt{\frac{\log 2/\delta}{2m}}}
    \end{align*}
    for any $\vec{\psi}_{\theta_W} \in \mathcal{H}_{\vec{\psi}_{\theta_W}}$, where hypothesis space $\mathcal{H}_1$ is defined as
    \begin{align}
        \mathcal{H}_1 = \{(w,a, z) \in \mathcal{W}\times \mathcal{A} \times \mathcal{Z} \mapsto \|\vec{\psi}_{\theta_W}(w) - \vec{g}(a, z) \|^2 \in \mathbb{R} ~|~ \vec{g} \in \mathcal{G},\vec{\psi}_{\theta_W} \in \mathcal{H}_{\vec{\psi}_{\theta_W}}\}, \label{eq:def-h1}
    \end{align}
    and $\kappa_1$ is the misspecificaiton error in the stage 1 regression defined as 
    \begin{align}
        \kappa_1 = \max_{\vec{\psi}_{\theta_W} \in \mathcal{H}_{\vec{\psi}_{\theta_W}}} \min_{\vec{g}\in\mathcal{G}} \expect[W,A,Z]{\|\vec{\psi}_{\theta_W}(W) - \vec{g}(A,Z)\|^2}. \label{eq:def-kappa1}
    \end{align}

\begin{proof}
    From Assumption~\ref{assume-consist}, we have
    \begin{align*}
        &\left\|\hatexpect[W|A,Z]{h(A,W)} - \expect[W|A,Z]{h(A,W)}\right\|_{P(A, Z)} \\
        &= \left\|\vec{u}^\top \left(\vec{\psi}_{\theta_{A(2)}}(a) \otimes \left(\expect[W|A,Z]{\vec{\psi}_{\theta_W}}(a,z) - \hatexpect[W|A,Z]{\vec{\psi}_{\theta_W}}(a,z)\right)\right)\right\|_{P(A,Z)}\\
        &= \sqrt{\expect[A,Z]{\left|\vec{u}^\top \left(\vec{\psi}_{\theta_{A(2)}}(a) \otimes \left(\expect[W|A,Z]{\vec{\psi}_{\theta_W}}(a,z) - \hatexpect[W|A,Z]{\vec{\psi}_{\theta_W}}(a,z)\right)\right)\right|^2}}\\
        &\leq \sqrt{\expect[A,Z]{M^2\left\| \expect[W|A,Z]{\vec{\psi}_{\theta_W}}(a,z) - \hatexpect[W|A,Z]{\vec{\psi}_{\theta_W}}(a,z)\right\|^2}}\\
        & = M \left\|\expect[W|A,Z]{\vec{\psi}_{\theta_W}}(a,z) - \hatexpect[W|A,Z]{\vec{\psi}_{\theta_W}}(a,z)\right\|_{P(A,Z)}
    \end{align*}
    
    From Proposition~\ref{prop:bound}, we have
    \begin{align*}
        &\expect[W,A,Z]{\|\vec{\psi}_{\theta_W}(W) - \hatexpect[W|A,Z]{\vec{\psi}_{\theta_W}(W)}(A,Z)\|^2}\\
        &\quad \leq \hat{\mathcal{L}}_1(\hatexpect[W|A,Z]{\vec{\psi}_{\theta_W}(W)}) + 2\hat{\mathfrak{R}}_{S_1}(\mathcal{H}_1) + 12 \sqrt{\frac{\log 2/\delta}{2m}}
    \end{align*}
    with probability $1-\delta$, since all functions $f \in \mathcal{H}_1$ satisfy $\|f\| \leq 4$ from $\|\vec{\psi}_{\theta_W}\| \leq 1, \|\vec{g}\| \leq 1.$ Now, let $\vec{g}^*_{\theta_W}$ be 
    \begin{align*}
        \vec{g}^*_{\theta_W} = \argmin_{\vec{g}\in\mathcal{G}} \expect[W,A,Z]{\|\vec{\psi}_{\theta_W}(W) - \vec{g}(A,Z)\|^2}.
    \end{align*}
    Then, again from Proposition~\ref{prop:bound}, we have
    \begin{align*}
        &\hat{\mathcal{L}}_1(\vec{g}^*_{\theta_W}) \leq \expect[W,A,Z]{\|\vec{\psi}_{\theta_W}(W) - \vec{g}^*_{\theta_W}(A,Z)\|^2} + 2\hat{\mathfrak{R}}_{S_1}(\mathcal{H}_1) + 12 \sqrt{\frac{\log 2/\delta}{2m}}
    \end{align*}
    From the optimality of $\hatexpect[W|A,Z]{\vec{\psi}_{\theta_W}(W)}$, we have $\hat{\mathcal{L}}_1(\vec{g}^*_{\theta_W}) \geq \hat{\mathcal{L}}_1(\hatexpect[W|A,Z]{\vec{\psi}_{\theta_W}(W)})$, thus
    \begin{align}
        &\expect[W,A,Z]{\|\vec{\psi}_{\theta_W}(W) - \hatexpect[W|A,Z]{\vec{\psi}_{\theta_W}(W)}(A,Z)\|^2} \nonumber\\
        &\quad \leq \expect[W,A,Z]{\|\vec{\psi}_{\theta_W}(W) - \vec{g}^*_{\theta_W}(A,Z)\|^2} + 4\hat{\mathfrak{R}}_{S_1}(\mathcal{H}_1) + 24 \sqrt{\frac{\log 2/\delta}{2m}}. \label{eq:tmp1}
    \end{align}
    holds with probability $1-2\delta$. Since we have 
    \begin{align*}
        &\expect[W,A,Z]{\|\vec{\psi}_{\theta_W}(W) - g(A,Z)\|^2} \\
        &=  \expect[W,A,Z]{\|\vec{\psi}_{\theta_W}(W) - \expect[W|A,Z]{\vec{\psi}_{\theta_W}(W)}\|^2} + \expect[A,Z]{\|g(A,Z) - \expect[W|A,Z]{\vec{\psi}_{\theta_W}(W)}\|^2}
    \end{align*}
    for all $\vec{g} \in \mathcal{G}$, by subtracting $\expect[W,A,Z]{\|\vec{\psi}_{\theta_W}(W) - \expect[W|A,Z]{\vec{\psi}_{\theta_W}(W)}\|^2}$ from both sides of \eqref{eq:tmp1}, we have
    \begin{align*}
        &\expect[A,Z]{\|\expect[W|A,Z]{\vec{\psi}_{\theta_W}(W)} - \hatexpect[W|A,Z]{\vec{\psi}_{\theta_W}(W)}(A,Z)\|^2}\\
        &\quad \leq \expect[A,Z]{\|\expect[W|A,Z]{\vec{\psi}_{\theta_W}(W)} - \vec{g}^*_{\theta_W}(A,Z)\|^2} + 4\hat{\mathfrak{R}}_{S_1}(\mathcal{H}_1) + 24 \sqrt{\frac{\log 2/\delta}{2m}}\\
        &\quad \leq \kappa_1 + 4\hat{\mathfrak{R}}_{S_1}(\mathcal{H}_1) + 24 \sqrt{\frac{\log 2/\delta}{2m}}
    \end{align*}
    By taking the square root of both sides, we have
    \begin{align*}
        \left\|\hatexpect[W|A,Z]{\vec{\psi}_{\theta_W}(W)} - \expect[W|A,Z]{\vec{\psi}_{\theta_W}(W) }\right\|_{P(A, Z)} \leq M\sqrt{\kappa_1 + 4\hat{\mathfrak{R}}_{S_1}(\mathcal{H}_1) + 24\sqrt{\frac{\log 2/\delta}{2m}}}
    \end{align*}
\end{proof}
\end{lem}

The generalization error for Stage 2 can be shown by  similar reasoning as in \citet[Lemma 2]{Xu2021Learning}.

\begin{lem}\label{lem:2}
    Under Assumption~\ref{assume-consist}, given stage 1 data $S_1 = \{(a_i, z_i, w_i)\}_{i=1}^m,$  stage 2 data $S_2 = \{(\tilde a_i, \tilde z_i, \tilde y_i)\}_{i=1}^n,$ and the estimated structural function $\hat{h} (a, z)=(\hat{\vec{u}})^\top (\vec{\psi}_{\hat{\theta}_{A(2)}}(a) \otimes \vec{\psi}_{\hat{\theta}_{W}}(w))$, then for any $\delta>0$, with at least probability $1-4\delta$, we have
    \begin{align*}
        &\left\|\expect[Y|A,Z]{Y} - \hatexpect[W|A,Z]{\hat{h}(A,W)}\right\|_{P(A,Z)} \\
        &\quad \leq \kappa_2 +  M\sqrt{\kappa_1 + 4\hat{\mathfrak{R}}_{S_1}(\mathcal{H}_2) + 24 \sqrt{\frac{\log 2/\delta}{2m}}} + \sqrt{4\hat{\mathfrak{R}}_{S_2}(\mathcal{H}_2) + 24M^2 \sqrt{\frac{\log 2/\delta}{2n}}},
    \end{align*}
    where $\mathcal{H}_1$ is defined in \eqref{eq:def-h1}, and $\mathcal{H}_2$ is defined as
    \begin{align}
        \mathcal{H}_2 = \{(y, a, z) &\in \mathcal{Y} \times\mathcal{A}\times \mathcal{Z}  \nonumber\\
        &\mapsto(y -  \vec{u}^\top (\vec{\psi}_{\theta_{A(2)}}(a) \otimes \vec{g}(a,z)))^2 \in \mathbb{R} ~|~  \vec{g} \in \mathcal{G}, \vec{u},  \theta_{A(2)}\}, \label{eq:def-h2}
    \end{align}
    and $\kappa_2$ is the misspecification error in Stage 2 defined as
    \begin{align}
        \kappa_2 = \min_{h\in \mathcal{H}_h} \left\|\expect[W|A,Z]{h(A,W)}(A,Z) - \expect[Y|A,Z]{Y}(A,Z)\right\|_{P(A,Z)}. \label{eq:def-kappa2}
    \end{align}
    \begin{proof}
         From Proposition~\ref{prop:bound}, we have
    \begin{align*}
        &\expect[Y,A,Z]{\left|Y - \hatexpect[W|A,Z]{\hat{h}(A,W)}(A,Z)\right|^2} \leq \hat{\mathcal{L}}_2(\hat{h}) + 2\hat{\mathfrak{R}}_{S_2}(\mathcal{H}_2) + 12M^2 \sqrt{\frac{\log 2/\delta}{2n}}
    \end{align*}
    with probabiltiy $1-\delta$, since all functions $f \in \mathcal{H}_2$ satisfy $\|f\| \leq 4M^2$ from $\|Y\| \leq M,  |\vec{u}^\top (\vec{\psi}_{\theta_{A(2)}}(a) \otimes \vec{g})|\leq M$. Let $\tilde h$ be
    \begin{align*}
       \tilde h = \argmin_{h \in \mathcal{H}_h} \left\|\expect[W|A,Z]{h(A,W)}(A,Z) - \expect[Y|A,Z]{Y}(A,Z)\right\|_{P(A,Z)}. 
    \end{align*}
    Again from Proposition~\ref{prop:bound}, we have
    \begin{align*}
        & \hat{\mathcal{L}}_2(\tilde{h}) \leq \expect[Y,A,Z]{\left|Y - \hatexpect[W|A,Z]{\tilde{h}(A,W)}(A,Z)\right|^2} + 2\hat{\mathfrak{R}}_{S_2}(\mathcal{H}_2) + 12M^2 \sqrt{\frac{\log 2/\delta}{2n}}.
    \end{align*}
    with probabiltiy $1-\delta$. From the optimality of $\hat{h}$, we have $\hat{\mathcal{L}}_2(\tilde{h}) \geq \hat{\mathcal{L}}_2(\hat{h})$, hence
    \begin{align*}
        &\expect[Y,A,Z]{\left|Y - \hatexpect[W|A,Z]{\hat{h}(A,W)}(A,Z)\right|^2} \\
        &\leq \expect[Y,A,Z]{\left|Y - \hatexpect[W|A,Z]{\tilde{h}(A,W)}(A,Z)\right|^2} + 4\hat{\mathfrak{R}}_{S_2}(\mathcal{H}_2) + 24M^2 \sqrt{\frac{\log 2/\delta}{2n}}
    \end{align*}
    By subtracting $\expect[Y,A,Z]{\|Y - \expect[Y|A,Z]{Y}\|^2}$ from both sides, we have
    \begin{align*}
        &\expect[A,Z]{\left|\expect[Y|A,Z]{Y} - \hatexpect[W|A,Z]{\hat{h}(A,W)}(A,Z)\right|^2} \\
        &\leq \expect[A,Z]{\left|\expect[Y|A,Z]{Y} - \hatexpect[W|A,Z]{\tilde{h}(A,W)}(A,Z)\right|^2} + 4\hat{\mathfrak{R}}_{S_2}(\mathcal{H}_2) + 24M^2 \sqrt{\frac{\log 2/\delta}{2n}}
    \end{align*}
    with probability $1-2\delta$. By taking the square root of both sides, with probability $1-2\delta$, we have
    \begin{align*}
        &\|\expect[Y|A,Z]{Y} - \hatexpect[W|A,Z]{\tilde{h}(A,W)}(A,Z)\|_{P(A,Z)}\\
        &\leq \sqrt{\expect[A,Z]{\left|\expect[Y|A,Z]{Y} - \hatexpect[W|A,Z]{\tilde{h}(A,W)}(A,Z)\right|^2} + 4\hat{\mathfrak{R}}_{S_2}(\mathcal{H}_2) + 24M^2 \sqrt{\frac{\log 2/\delta}{2n}}}\\
        &\leq \|\expect[Y|A,Z]{Y} - \hatexpect[W|A,Z]{\tilde{h}(A,W)}(A,Z)\|_{P(A,Z)} + \sqrt{4\hat{\mathfrak{R}}_{S_2}(\mathcal{H}_2) + 24M^2 \sqrt{\frac{\log 2/\delta}{2n}}}\\
        &\overset{(a)}{\leq} \left\|\expect[W|A,Z]{\tilde{h}(A,W)}(A,Z) - \hatexpect[W|A,Z]{\tilde{h}(A,W)}(A,Z)\right\|_{P(A,Z)} \\
        &\quad\quad + \left\|\expect[W|A,Z]{\tilde{h}(A,W)}(A,Z) - \expect[Y|A,Z]{Y}(A,Z)\right\|_{P(A,Z)} +  \sqrt{4\hat{\mathfrak{R}}_{S_2}(\mathcal{H}_2) + 24M^2 \sqrt{\frac{\log 2/\delta}{2n}}}\\
        &\overset{(b)}{\leq} \kappa_2 +  M\sqrt{\kappa_1 + 4\hat{\mathfrak{R}}_{S_1}(\mathcal{H}_2) + 24 \sqrt{\frac{\log 2/\delta}{2m}}} + \sqrt{4\hat{\mathfrak{R}}_{S_2}(\mathcal{H}_2) + 24M^2 \sqrt{\frac{\log 2/\delta}{2n}}} 
   \end{align*}
    where (a) holds from the triangular inequality and (b) holds from Lemma~\ref{lem:1}.
    \end{proof}
\end{lem}

Given the generalization errors in both stages, we can bound the error in \eqref{eq:functional} for the estimated bridge function $\hat{h}$. We need Assumption~\ref{assum:completeness-proxy} to connect the error in \eqref{eq:functional} and the error $\|\hat{h} - h^*\|_{\prob{A,W}}$. Let us restate Assumption~\ref{assum:completeness-proxy}.

\completenessproxy*
Given this assumption, we can consider the following constant $\tau_a$.
\begin{align*}
    \tau_a = \max_{h \in L^2(\prob{W|A=a}), h\neq h^*} \frac{\|h^*(a,W) - h(a,W)\|_{\prob{W|A=a}}}{\|\expect[W|A=a,Z]{h^*(a,W)} - \expect[W|A=a,Z]{h(a,W)}\|_{\prob{Z|A=a}}}
\end{align*}
Note that Assumption~\ref{assum:completeness-proxy} ensures $\tau_a < \infty$. We can bound the error using the supremum of this constant, $\tau = \sup_a \tau_a$.

 \begin{thm}\label{thm:totalConsistency}
    Let Assumptions~\ref{assum:completeness-proxy}~and~\ref{assume-consist} hold. Given stage 1 data $S_1 = \{(w_i, a_i, z_i)\}_{i=1}^m$ and stage 2 data $S_2 = \{(\tilde y_i, \tilde a_i, \tilde z_i)\}_{i=1}^n$, for any $\delta>0$, with at least probability of $1-6\delta$, we have
    \begin{align*}
        &\|h^*(A,W) - \hat{h}(A,W)\|_{P(A,W)} \\
        &\quad \leq \tau \left(\kappa_2 + 2M\sqrt{\kappa_1 + 4\hat{\mathfrak{R}}_{S_1}(\mathcal{H}_1) + 24\sqrt{\frac{\log 2/\delta}{2m}}} + \sqrt{4\hat{\mathfrak{R}}_{S_2}(\mathcal{H}_2) + 24M^2\sqrt{\frac{\log 2/\delta}{2n}}}\right),
    \end{align*}
    where $\tau = \sup_a \tau_a$ and $\mathcal{H}_1, \kappa_1, \mathcal{H}_2, \kappa_2$ are defined in \eqref{eq:def-h1}, \eqref{eq:def-kappa1}, \eqref{eq:def-h2}, \eqref{eq:def-kappa2}, respectively.
\end{thm}

\begin{proof}
\begin{align*}
&\|h^*(A,W) - \hat{h}(A,W)\|_{P(A,W)}\\
&\leq\sqrt{\expect[A]{\|h^*(a,W) - \hat{h}(a,W)\|^2_{\prob{W|A=a}}}}\\
&\leq\sqrt{\expect[A]{\tau_a^2 \left\|\expect[W|A=a,Z]{h^*(a,W) - \hat{h}(a,W)}\right\|^2_{\prob{Z|A=a}}}}\\
&\leq \tau \sqrt{\expect[A]{\left\|\expect[W|A=a,Z]{h^*(a,W) - \hat{h}(a,W)}\right\|^2_{\prob{Z|A=a}}}}\\
    &\leq\tau\left\|\expect[Y|A,Z]{Y} - \expect[W|A,Z]{\hat{h}(A,W)}\right\|_{\prob{A,Z}}\\
    &\leq\tau\left \|\expect[Y|A,Z]{Y} - \hatexpect[W|A,Z]{\hat{h}(A,W)}\right\|_{\prob{A,Z}} + \tau\left\|\expect[W|A,Z]{\hat{h}(A,W)} - \hatexpect[W|A,Z]{\hat{h}(A,W)}\right\|_{\prob{A,Z}}
\end{align*}
Using Lemmas~\ref{lem:1}~and~\ref{lem:2}, the result thus follows.
\end{proof}

From this result, we obtain the following corollary.

\begin{corollary}
Let Assumption~\ref{assume-consist} hold and $\kappa_1, \kappa_2 = 0$. If $\hat{\mathfrak{R}}_{S_1}(\mathcal{H}_1) \to 0$ and $\hat{\mathfrak{R}}_{S_2}(\mathcal{H}_2) \to 0$ in probability as the dataset size increases, $\hat{h}$ converges to ${h^*}$ in probability with respect to $\|\cdot\|_{\prob{A,W}}$.
\end{corollary}

\subsection{Consistency Result for Causal Parameters}
In this section, we develop a consistency result for causal parameters discussed in the main body. First, we consider the structural function. Given estimated bridge function $\hat{h}(a,w) = \hat{\vec{u}}^\top (\vec{\psi}_{\hat{\theta}_{A(2)}}(a) \otimes \vec{\psi}_{\hat{\theta}_W}(w))$, we estimate the structural function by taking the empirical mean over $W$. To make the discussion simple, we assume the access to an additional data sample $S_W = \{w^{\mathrm{extra}}_i\}_{i=1}^{n_W}$, such that the estimated structural function is given as
\begin{align}
    \hat{f}_{\mathrm{struct}}(a) =\hat{\vec{u}}^\top (\vec{\psi}_{\hat{\theta}_{A(2)}}(a) \otimes \vec{\mu}_{\hat{\theta}_W}), \label{eq:f-hat-def}
\end{align}
where 
\begin{align*}
    \vec{\mu}_{\hat{\theta}_W} = \sum_{i=1}^{n_W} \vec{\psi}_{\hat{\theta}_W}(w^{\mathrm{extra}}_i).
\end{align*}
Note that empirically, we can use outcome-proxy data in $S_1$ instead of $S_W$.

To bound  the deviation from the true structural function, we need the following assumption.
\begin{assum}\label{assume-good-support}
   Let $\rho_W(w), \rho_{W|A}(w|a)$ be the respective density functions of probability distributions $\prob{W}, \prob{W|A}$. The densities satisfy
   \begin{align*}
        \eta_a \coloneqq \left\| \frac{\rho_W(W)}{\rho_{W|A}(W|a)} \right\|_{\prob{W|A=a}} < \infty.
   \end{align*}
\end{assum}

Given Assumption~\ref{assume-good-support}, we can bound  the error in structural function estimation. Before stating the theorem, let us introduce a useful concentration inequality for multi-dimensional random variables. \footnote{Lemma~\ref{lem:multi-hoeffding} is discussed in MathOverflow (\url{https://mathoverflow.net/questions/186097/hoeffdings-inequality-for-vector-valued-random-variables}, accessed on June 2nd 2021)}

\begin{lem}\label{lem:multi-hoeffding}
    Let $\vec{x}_1, \dots, \vec{x}_n \in [-1, 1]^d$ be independent random variables. Then, with probability at least $1-\delta$, we have 
    \begin{align*}
        \left\|\frac{1}{n} \sum_{i=1}^n \vec{x}_i - \vec{\mu} \right\| \leq \sqrt{\frac{2d \log 2d/\delta}{n}}
    \end{align*}
    where $\vec{\mu} = \expect{\frac{1}{n} \sum_{i=1}^n \vec{x}_i}$.
    \begin{proof}
        Let $j$-th coordinate of $\vec{x}_i$ be denoted as $(\vec{x}_i)_j$. Then,
        \begin{align*}
            &\prob{\left\|\frac{1}{n} \sum_{i=1}^n \vec{x}_i - \vec{\mu} \right\| \leq \varepsilon} \\
            &=\prob{ \sum_{j=1}^d \left(\frac{1}{n} \sum_{i=1}^n (\vec{x}_i)_j - (\vec{\mu})_j\right)^2  \leq \varepsilon^2 }\\
            &\leq \prob{ \bigcap_{j=1}^d \left|\frac{1}{n} \sum_{i=1}^n (\vec{x}_i)_j - (\vec{\mu})_j\right|  \leq \frac{\varepsilon}{\sqrt{d}} }\\
            &\leq 1 - 2d \exp\left(-\frac{n\varepsilon^2}{2d}\right),
        \end{align*}
        Where the last inequality holds from Hoeffding's inequality. The claim follows by solving $\delta = d \exp\left(-\frac{n\varepsilon^2}{2d}\right)$.
    \end{proof}
\end{lem}
Given this lemma, we can prove the following result.
\begin{thm}\label{thm:consistency-structural}
   Let Assumptions~\ref{assum:completeness-proxy},~\ref{assume-consist} and \ref{assume-good-support} hold. 
   Let $\eta = \sup_{a\in \mathcal{A}} \eta_a$. 
   Given stage 1 data $S_1 = \{(w_i, a_i, z_i)\}_{i=1}^m$, stage 2 data $S_2 = \{(\tilde y_i, \tilde a_i, \tilde z_i)\}_{i=1}^n$, additional outcome-proxy variable data $S_W = \{w^{\mathrm{extra}}_i\}_{i=1}^{n_W}$ , then with  probability  at least $1-7\delta$, we have
   \begin{align*}
       &\|f^*_{\mathrm{struct}} - \hat{f}_{\mathrm{struct}} \|_{\prob{A}}\\
       &\leq \sqrt{\frac{2d_W \log 2d_W/\delta}{n_W}} + \eta\tau \left(\kappa_2 + 2M\sqrt{\kappa_1 + 4\hat{\mathfrak{R}}_{S_1}(\mathcal{H}_1) + 24\sqrt{\frac{\log 2/\delta}{2m}}} + \sqrt{4\hat{\mathfrak{R}}_{S_2}(\mathcal{H}_2) + 24M^2\sqrt{\frac{\log 2/\delta}{2n}}}\right),
   \end{align*}
   where $\hat{f}_{\mathrm{struct}}$ is given in \eqref{eq:f-hat-def} and $d_W$ is the dimension of $\vec{\psi}_{\theta_W}$.
   \begin{proof}
   From the relationship between structural function and bridge function, we have
   \begin{align*}
       &\|f^*_{\mathrm{struct}} - \hat{f}_{\mathrm{struct}} \|_{\prob{A}}\\
       & = \|\expect[W]{h^*(A,W)} -  \hat{f}_{\mathrm{struct}}\|_{\prob{A}}\\
       & \leq  \left\|\expect[W]{h^*(A,W)} -  \expect[W]{\hat{h}(A,W)}\right\|_{\prob{A}} + \left\| \expect[W]{\hat{h}(A,W)} - \hat{f}_{\mathrm{struct}}\right\|_{\prob{A}}.
   \end{align*}
   We can bound each term as follows. For the first term, we have
   \begin{align*}
       &\left\|\expect[W]{h^*(\cdot,W)} -  \expect[W]{\hat{h}(\cdot,W)}\right\|_{\prob{A}}\\
       & = \sqrt{\int \left(\expect[W]{h^*(a,W) - \hat{h}(a,W)}\right)^2 \rho_{A}(a)\intd a}\\
       & = \sqrt{\int \left(\expect[W|A=a]{(h^*(a,W) - \hat{h}(a,W))\frac{\rho_W(W)}{\rho_{W|A}(W|a)}}\right)^2 \rho_{A}(a)\intd a}\\
       & \leq \sqrt{\int \left\|h^*(a,W) - \hat{h}(a,W)\right\|^2_{\prob{W|A=a}}\left\|\frac{\rho_W(W)}{\rho_{W|A}(W|a)}\right\|^2_{\prob{W|A=a}} \rho_{A}(a)\intd a}\quad \because \text{~Cauchy–Schwarz inequality}\\
       & \leq \sqrt{\int \eta^2_a \expect[W|A=a]{(h^*(a,W) - \hat{h}(a,W))^2} \rho_{A}(a)\intd a}\quad \because \text{~Assumption~\ref{assume-good-support}}\\
       & \leq \eta \sqrt{\expect[W,A]{(h^*(A,W) - \hat{h}(A,W))^2}}\\
       &= \eta \|h^*- \hat{h}\|_{\prob{A,W}}.
   \end{align*}
   From Theorem~\ref{thm:totalConsistency}, with  probability at least $1-6\delta$, we have
   \begin{align*}
       &\left\|\expect[W]{h^*(\cdot,W)} -  \expect[W]{\hat{h}(\cdot,W)}\right\|_{\prob{A}} \\
       & \leq \eta\tau \left(\kappa_2 + 2M\sqrt{\kappa_1 + 4\hat{\mathfrak{R}}_{S_1}(\mathcal{H}_1) + 24\sqrt{\frac{\log 2/\delta}{2m}}} + \sqrt{4\hat{\mathfrak{R}}_{S_2}(\mathcal{H}_2) + 24M^2\sqrt{\frac{\log 2/\delta}{2n}}}\right).
   \end{align*}
   
   For the second term, from Assumption~\ref{assume-consist}, we have with  probability at least $1-\delta$,
   \begin{align*}
       \left\| \expect[W]{\hat{h}(A,W)} - \hat{f}_{\mathrm{struct}}\right\|_{\prob{A}}
       & \leq M \left\|\vec{\mu}_{\hat{\theta}_W} -\expect[W]{\vec{\psi}_{\hat{\theta}_W}}\right\|\\
       &\leq \sqrt{\frac{2d_W \log 2d_W/\delta}{n_W}}
   \end{align*}
   The last inequality holds from Lemma~\ref{lem:multi-hoeffding}. Using the uniform inequality, we have shown the claim.
   \end{proof}
\end{thm}

We can evaluate the error in estimating a value function as well. Given $S_3 = \{\check{w}_i, \check{c}_i\}_{i=1}^{n'}$, we estimate the value function as 
\begin{align*}
    \hat{v}(\pi) = \frac1{n'} \sum_{i=1}^{n}  \hat{\vec{u}}^\top (\vec{\psi}_{\hat{\theta}_{A(2)}}(\pi(\check{c}_i)) \otimes \vec{\psi}_{\hat{\theta}_{W}}(\check{w}_i)),
\end{align*}
Furthermore, we assume the following relationship between distributions of $A$ and $C$
\begin{assum} \label{assume-policy-l2}
    There exists a constant $\sigma$ such that 
    \begin{align*}
        \|l(\pi(C), W)\|_{\prob{C,W}} \leq  \sigma \|l(A,W)\|_{\prob{A,W}} 
    \end{align*}
    for all square integrable functions $l: \mathcal{A} \times \mathcal{W} \to \mathbb{R}, \|l\| < \infty$.
\end{assum}

Given these assumptions, we have the following theorem.

\begin{thm}
   Let Assumptions~\ref{assum:completeness-proxy},~\ref{assume-consist},~\ref{assume-good-support},~\ref{assume-policy-l2} hold. Given stage 1 data $S_1 = \{(w_i, a_i, z_i)\}_{i=1}^m$, stage 2 data $S_2 = \{(\tilde y_i, \tilde a_i, \tilde z_i)\}_{i=1}^n$, stage 3 data $S_3= \{\check{w}_i, \check{c}_i\}_{i=1}^{n'}$, with at least probability $1-7\delta$, we have
   \begin{align*}
       &|v(\pi) - \hat{v}(\pi)|\\
       &\leq\sigma\tau \left(\kappa_2 + 2M\sqrt{\kappa_1 + 4\hat{\mathfrak{R}}_{S_1}(\mathcal{H}_1) + 24\sqrt{\frac{\log 2/\delta}{2m}}} + \sqrt{4\hat{\mathfrak{R}}_{S_2}(\mathcal{H}_2) + 24M^2\sqrt{\frac{\log 2/\delta}{2n}}}\right)\\
       &\quad\quad  + \sqrt{\frac{M^2\log 2/\delta}{n'}}.
   \end{align*}
  
   \begin{proof}
        We have
        \begin{align*}
    |v(\pi) - \hat{v}(\pi)| &= \left| \expect[W|C]{h^*(\pi(C), W)} - \hat{v}(\pi)\right| \\
    &\leq \left| \expect[W,C]{h(\pi(C), W)} - \expect[W,C]{\hat{h}(\pi(C), W)}\right| \\
    &\quad \quad + \left| \expect[W, C]{\hat{h}(\pi(C), W)} - \hat{v}(\pi) \right|.
\end{align*}
    We bound each term as follows. For the first term, we have
    \begin{align*}
        &\left| \expect[W,C]{h^*(\pi(C), W)} - \expect[W,C]{\hat{h}(\pi(C), W)}\right|\\
        &\leq \|h^*(\pi(C), W) -\hat{h}(\pi(C), W) \|_{\prob{C,W}} \quad\because\text{Jensen's inequality}\\
        &\leq \sigma \|h^*(A, W) -\hat{h}(A, W) \|_{\prob{A,W}} \quad\because\text{Assumption~\ref{assume-policy-l2}}
    \end{align*}
    Hence, from Theorem~\ref{thm:totalConsistency}, we have
    \begin{align*}
        &\left| \expect[W,C]{h^*(\pi(C), W)} - \expect[W,C]{\hat{h}(\pi(C), W)}\right|\\
        &\leq \sigma\tau \left(\kappa_2 + 2M\sqrt{\kappa_1 + 4\hat{\mathfrak{R}}_{S_1}(\mathcal{H}_1) + 24\sqrt{\frac{\log 2/\delta}{2m}}} + \sqrt{4\hat{\mathfrak{R}}_{S_2}(\mathcal{H}_2) + 24M^2\sqrt{\frac{\log 2/\delta}{2n}}}\right)
    \end{align*}
    with probability at least $1-6\delta$. 
    
    For the second term, we have 
    \begin{align*}
        &\left|\expect[W,C]{\hat{\vec{u}}^\top \left(\vec{\psi}_{\hat{\theta}_{A(2)}}(\pi(C)) \otimes \vec{\psi}_{\hat{\theta}_{W}}(W)\right)}- \hat{v}(\pi)\right| \leq \sqrt{\frac{M^2\log 2/\delta}{n'}}
    \end{align*}
    with probability at least  $1-\delta$, from Hoeffding's inequality. The result is obtained by taking the union bound.
   \end{proof}
\end{thm}

\section{DFPV algorithm with observable confounders}
\label{sec:observable-confounder}
In this appendix, we formulate the DFPCL method when observable confounders are present, building on \citet{Tchetgen2020} and \citetMastouri. Here, we consider the causal graph given in Figure~\ref{fig:observable_confounder}.  In addition to variables $(A, Y, Z, W)$, we have an observable confounder $X \in \mathcal{X}$. The structural function $f_\mathrm{struct}$ we aim to learn is 
\begin{align*}
    f_\mathrm{struct}(a) = \expect[X,U]{\expect{Y \mid A =a, X, U}}.
\end{align*}

\begin{figure}
    \centering
        \begin{tikzpicture}
           \node[state, fill=yellow] (a) at (0,0) {$A$};
    
        \node[state, fill=yellow] (y) [right =of a, xshift = 0.3cm] {$Y$};
        \node[state, fill=yellow] (z) [above left =of a, xshift = 0.7cm, yshift=-0.3cm] {$Z$};
        \node[state, fill=yellow] (w) [above right =of y, xshift = -0.5cm, yshift=-0.3cm] {$W$};
        \node[state] (eps) [above right =of a, xshift = -0.5cm, yshift=1.0cm] {$U$};
        \node[state, fill=yellow] (x) [above right =of a, xshift = -0.5cm, yshift=-0.3cm] {$X$};
        \path (a) edge (y);
        \path[bidirected] (z) edge (a);
        \path (eps) edge (y);
        \path (eps) edge (a);
        \path (eps) edge (z);
        \path (eps) edge (x);
        \path[bidirected] (eps) edge (w);
        \path (w) edge (y);
        \path (x) edge (y);
        \path (x) edge (a);
        \path (x) edge (z);
        \path (x) edge (w);
        \end{tikzpicture}   
    \caption{Causal graph with observable confounder}
    \label{fig:observable_confounder}
\end{figure}

The structural assumption and completeness assumption including observable confounders are given as follows.
\begin{assum}[Structural Assumption \citep{Mastouri2021Proximal}] \label{assum:stuctural-with-observable}
We assume $Y \indepe Z | A, U, X$, and $W \indepe (A,Z) | U, X$.
\end{assum}
\begin{assum}[Completeness Assumption \citep{Mastouri2021Proximal}] \label{assum:completeness-with-observable}
Let $l: \mathcal{U} \to \mathbb{R}$ be any square integrable function $\|l\|_{\prob{U}}$. We assume the following:
\begin{align*}
    &\expect{l(U) \mid A=a, Z=z, X=x}=0\quad \forall (a,z,x) \in \mathcal{A} \times \mathcal{Z} \times \mathcal{X} \quad\Leftrightarrow \quad l(u) = 0~\mathrm{a.s.}\\
    &\expect{l(U) \mid A=a, Z=z, X=x}=0\quad \forall (a,z,x) \in \mathcal{A} \times \mathcal{Z}\times \mathcal{X} \quad\Leftrightarrow \quad l(u) = 0~\mathrm{a.s.}
\end{align*}
\end{assum}
Following  similar reasoning as in Section \ref{sec:preliminary}, we can estimate the bridge function $\hat{h}: \mathcal{A} \times \mathcal{X} \times \mathcal{W} \to \mathbb{R}$ by minimizing the following loss:
\begin{align*}
    \hat{h} = \argmin_{h \in \mathcal{H}_h} \tilde{\mathcal{L}}(h), \quad \tilde{\mathcal{L}}(h) = \expect[Y,A,Z, X]{(Y - \expect[W|Z,A,X]{h(A,X,W)})^2} + \Omega(h). 
\end{align*}
Given bridge function, we can estimate the structural function by
\begin{align*}
    f_{\mathrm{struct}}(a) = \expect[X,W]{h^*(a, X, W)}.
\end{align*}
Similar to \citetMastouri,
we model 
\begin{align*}
&\expect[W|a,z]{\vec{\psi}_{\theta_{W}}(w)} = \vec{V}(\vec{\phi}_{\theta_{A(1)}}(A) \otimes \vec{\phi}_{\theta_Z}(Z)) \otimes \vec{\phi}_{\theta_{X(1)}}(X))\\
&  h(a,x, w) = \vec{u}^\top (\vec{\psi}_{\theta_{A(2)}}(a) \otimes \vec{\psi}_{\theta_{X(2)}}(x) \otimes \vec{\psi}_{\theta_{W}}(w)),
\end{align*}
where $\vec{\phi}_{\theta_{X(1)}}(X), \vec{\psi}_{\theta_{X(2)}}(X)$ are the feature maps of $X$ parameterized by $\theta_{X(1)}, \theta_{X(2)}$, respectively. Then, in stage 1, we learn $(\vec{V}, \theta_{A(1)}, \theta_{Z}, \theta_{X(1)})$ by minimizing 
\begin{align}
     \hat{\mathcal{L}}_1(\vec{V}, \theta_{A(1)}, \theta_Z, \theta_{X(1)}) = \frac1m \sum_{i=1}^m \left\|\vec{\psi}_{\theta_{W}}(w_i) - \vec{V}\left(\vec{\phi}_{\theta_{A(1)}}(a_i) \otimes \vec{\phi}_{\theta_Z}(z_i) \otimes \vec{\phi}_{\theta_X}(x_i) \right) \right\|^2 + \lambda_1 \|\vec{V}\|^2
,  \label{eq:stage1-loss-with-confounder}
\end{align}
which estimates the conditional expectation $\expect[W]{\vec{\psi}_W(W) |A=a, Z=z, X=x}$.  Let $(\hat{\vec{V}}, \hat{\theta}_{A(1)}, \hat{\theta}_{Z}, \hat{\theta}_{X(1)})$ be the minimizer of $\hat{\mathcal{L}}_1$.
Then, in stage 2, we can learn $(\vec{u}, \theta_{A(2)}, \theta_{Z}, \theta_{X(2)})$ using
\begin{align}
    \hat{\mathcal{L}}_2(\vec{u}, \theta_{A(2)}, \theta_{W}, \theta_{X(2)}) = \frac1n \sum_{i=1}^n \left(\tilde y_i - \vec{u}^\top \left(\vec{\psi}_{\theta_{A(2)}}(\tilde a_i) \otimes \vec{\psi}_{\theta_{X(2)}}(\tilde x_i) \otimes \hat{\vec{V}}\vec{v}_1(\tilde a_i,\tilde x_i,\tilde z_i)  \right)\right)^2 + \lambda_2 \|\vec{u}\|^2, \label{eq:stage2-loss-with-confounder}
\end{align}
where we denote $\vec{v}_1(a,x,z) = \left(\vec{\phi}_{\hat{\theta}_{A(1)}}(\tilde a_i) \otimes\vec{\phi}_{\hat{\theta}_X}(\tilde x_i) \otimes\vec{\phi}_{\hat{\theta}_Z}(\tilde z_i)\right)$.

In DFPV, we first fix parameters $\vec{\theta} = (\theta_{A(1)}, \theta_{Z}, \theta_{X(1)}, \theta_{A(2)}, \theta_{Z}, \theta_{X(2)})$ and obtain weights. This is given as 
\begin{align}
    &\hat{\vec{V}}(\vec{\theta}) = \Psi_1^\top \Phi_1 (\Phi_1^\top \Phi_1 + m\lambda_1 I)^{-1}, &
    &\hat{\vec{u}}(\vec{\theta}) = \left(\Phi_2 ^\top \Phi_2 + n\lambda_2 I\right)^{-1}\Phi_2^\top \vec{y}_2, \label{eq:weights-sol-with-confounder}
\end{align}
where we denote $\vec{\theta} = (\theta_{A(1)}, \theta_Z, \theta_{A(2)}, \theta_{
W})$ and define matrices as follows:
\begin{align*}
    &\Psi_1 = \left[\vec{\psi}_{\theta_W}(w_1), \dots, \vec{\psi}_{\theta_W}(w_m)\right]^\top, & &\Phi_1 = [\vec{v}_1(a_1, z_1), \dots, \vec{v}_1(a_m, z_m)]^\top,\\
    &\vec{y}_2 = [\tilde y_1, \dots, \tilde y_n]^\top, & &\Phi_2 = [\vec{v}_2(\tilde a_1, \tilde z_1), \dots, \vec{v}_2(\tilde a_n, \tilde z_n)]^\top,&\\
    &\vec{v}_1(a, x, z) = \vec{\phi}_{\theta_{A(1)}}(a) \otimes \vec{\phi}_{\theta_{X(1)}}(x) \otimes\vec{\phi}_{\theta_Z}(z), &&\vec{v}_2(a, z) = \vec{\psi}_{\theta_{A(2)}}(a) \otimes \vec{\psi}_{\theta_{X(2)}}(x) \otimes \left(\hat{\vec{V}}(\vec{\theta})\vec{v}_1(a, x, z)\right).
\end{align*}
We learn the parameters by minimizing the following:
\begin{align*}
    &\hat{\mathcal{L}}^{\mathrm{DFPV}}_1(\vec{\theta}) = \frac1m \sum_{i=1}^m \left\|\vec{\psi}_{\theta_{W}}(w_i) - \hat{\vec{V}}(\vec{\theta})\vec{v}_1(a_i, x_i, z_i) \right\|^2 + \lambda_1 \|\hat{\vec{V}}(\vec{\theta})\|^2,\\
    &\hat{\mathcal{L}}^{\mathrm{DFPV}}_2(\vec{\theta}) = \frac1n \sum_{i=1}^n \left(\tilde y_i - \hat{\vec{u}}(\vec{\theta})^\top \vec{v}_2(\tilde a_i, \tilde x_i, \tilde z_i) \right)^2 \!+\!  \lambda_2 \|\hat{\vec{u}}(\vec{\theta})\|^2.
\end{align*}
The algorithm is given in Algorithn~\ref{algo2}.

\begin{algorithm}
    \caption{Deep Feature Instrumental Variable with Observable Confounder}
    \begin{algorithmic}[1]  \label{algo2}
    \renewcommand{\algorithmicrequire}{\textbf{Input:}}
    \renewcommand{\algorithmicensure}{\textbf{Output:}}
   \REQUIRE Stage 1 data $(a_i, z_i, w_i, x_i)$, Stage 2 data $(\tilde a_i \tilde z_i, \tilde y_i, \tilde x_i)$, Regularization parameters $(\lambda_1, \lambda_2)$. Initial values $\vec{\theta}^{(0)} = (\theta^{(0)}_{A(1)}, \theta^{(0)}_Z, \theta^{(0)}_{A(2)}, \theta^{(0)}_{W},\theta^{(0)}_{X(1)}, \theta^{(0)}_{X(2)})$. Learning rate $\alpha$, additional data $(x_i^{\mathrm{extra}}, w_i^{\mathrm{extra}})$
 \ENSURE Estimated structural function $\hat{f}_\mathrm{struct}(a)$
  \STATE $t \leftarrow 0$
 \REPEAT
    \STATE Compute $\hat{\vec{V}}(\vec{\theta}), \hat{\vec{u}}(\vec{\theta})$ in \eqref{eq:weights-sol-with-confounder} 
    \STATE Update parameters in features $\vec{\theta}^{(t+1)} \leftarrow (\theta^{(t+1)}_{A(1)}, \theta^{(t+1)}_Z, \theta^{(t+1)}_{A(2)}, \theta^{(t+1)}_{W},\theta^{(t+1)}_{X(1)}, \theta^{(t+1)}_{X(2)})$ by
    \begin{align*}
        &\theta^{(t+1)}_{A(1)} \leftarrow \theta^{(t)}_{A(1)} - \alpha \nabla_{\theta_{A(1)}}\hat{\mathcal{L}}^{\mathrm{DFPV}}_1(\vec{\theta})|_{\vec{\theta} =\vec{\theta}^{(t)}},& & \theta^{(t+1)}_{Z} \leftarrow \theta^{(t)}_{Z} - \alpha \nabla_{\theta_{Z}}\hat{\mathcal{L}}^{\mathrm{DFPV}}_1(\vec{\theta})|_{\vec{\theta} =\vec{\theta}^{(t)}}\\
        &\theta^{(t+1)}_{X(1)} \leftarrow \theta^{(t)}_{X(1)} - \alpha \nabla_{\theta_{X(1)}}\hat{\mathcal{L}}^{\mathrm{DFPV}}_1(\vec{\theta})|_{\vec{\theta} =\vec{\theta}^{(t)}}, & & \theta^{(t+1)}_{X(2)} \leftarrow \theta^{(t)}_{X(2)} - \alpha \nabla_{\theta_{X(2)}}\hat{\mathcal{L}}^{\mathrm{DFPV}}_2(\vec{\theta})|_{\vec{\theta} =\vec{\theta}^{(t)}}\\
        &\theta^{(t+1)}_{A(2)} \leftarrow \theta^{(t)}_{A(2)} - \alpha \nabla_{\theta_{A(2)}}\hat{\mathcal{L}}^{\mathrm{DFPV}}_2(\vec{\theta})|_{\vec{\theta} =\vec{\theta}^{(t)}}, & & \theta^{(t+1)}_{W} \leftarrow \theta^{(t)}_{W} - \alpha \nabla_{\theta_{W}}\hat{\mathcal{L}}^{\mathrm{DFPV}}_2(\vec{\theta})|_{\vec{\theta} =\vec{\theta}^{(t)}} 
    \end{align*}
    \STATE Increment counter $t \leftarrow t+1;$
  \UNTIL{\textbf{convergence}}
  \STATE Compute $\hat{\vec{u}}(\vec{\theta}^{(t)})$ from \eqref{eq:weights-sol-with-confounder} 
  \STATE Compute mean feature for $W$ using stage 1 dataset 
  \begin{align*}
    \vec{\mu}_{\theta_{X(2)} \otimes \theta_W} \leftarrow \frac1n \sum \vec{\psi}_{\theta^{(t)}_{X(2)}}(x^{\mathrm{extra}}_i) \otimes \vec{\psi}_{\theta^{(t)}_W}(w^{\mathrm{extra}}_i)    
  \end{align*}
 \RETURN $\hat{f}_\mathrm{struct}(a) = (\hat{\vec{u}}^{(t)})^\top \left(\vec{\psi}_{\hat{\theta}^{(t)}_{A(2)}}(a) \otimes \vec{\mu}_{\theta_{X(2)} \otimes \theta_W}\right)$
    \end{algorithmic} 
    \end{algorithm}

\section{Additional Experiments} \label{sec:additional-experiments}
In this appendix, we report the results of two additional experiments. One is a synthetic setting introduced in \citet{Mastouri2021Proximal}, which has a simpler data generating process. The other is based on the real-world setting introduced by  \citet{Deaner2018Proxy}.  In both setting, DFPV performs similarly to or better than existing methods.

\subsection{Experiments with Simpler Data Generating Process}
Here, we show the result for the synthetic setting proposed in \citet{Mastouri2021Proximal}. The data generating process for each variable is given as follows: 
\begin{align*}
    &U := [U_1, U_2], \quad U_2 \sim \mathrm{Unif}[-1, 2] \quad U_1, \sim \mathrm{Unif}[0, 1] - \mathbbm{1}[0 \leq U_2 \leq 1]\\
    &Z := [U_1 + \mathrm{Unif}[-1, 1],~U_2 + \mathcal{N}(0, 3)]\\
    &W := [U_1 + \mathcal{N}(0, 3),~ U_2 + \mathrm{Unif}[-1, 1]]\\
    &A := U_2 + \mathcal{N}(0, 0.05)\\
    &Y := U_2 \cos(2(A + 0.3U_1 + 0.2))
\end{align*}

From observations of $(Y, W, Z, A)$, we estimate $\hat{f}_\mathrm{struct}$ by PCL. For each estimated $\hat{f}_\mathrm{struct}$, we measure out-of-sample error as the mean square error of $\hat{f}$ versus true $f_\mathrm{struct}$ obtained from Monte-Carlo simulation. Specifically, we consider 20 evenly spaced values of $A \in [0.0, 1.0]$ as the test data. The results with data size $n=m=\{500, 1000\}$ are shown in Figure~\ref{fig:kpv_ate}.

\begin{figure}[t]
    \centering
    \includegraphics[height=150pt]{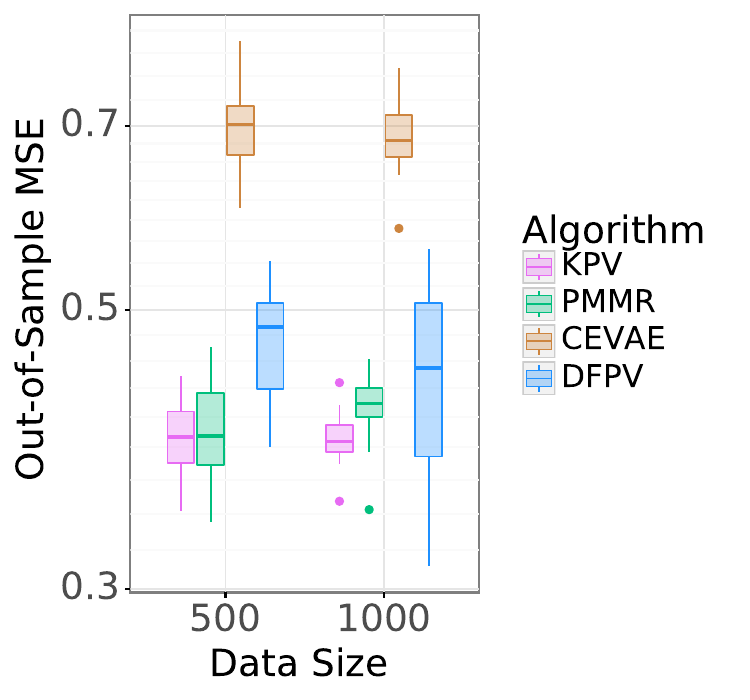}
    \caption{Result of structural function experiment in the setting in \citet{Mastouri2021Proximal}}
    \label{fig:kpv_ate}
\end{figure}

From Figure~\ref{fig:kpv_ate}, we can see that DFPV and CEVAE methods perform worse and have larger variances than KPV and PMMR methods. This is not surprising, since DFPV tends to require more data than KPV and PMMR, as needed to learn the neural net feature maps (rather than using fixed pre-defined kernel features). Hence, we can say that we should favor KPV and PMMR over DFPV when the data is low-dimensional and the relations between the variables are smooth. We would like to note, however, DFPV outperforms CEVAE, which shows that the proxy setting is still required.

\subsection{Experiments using Grade Retention dataset}
To test the performance of DFPV in a more realistic setting, we conducted the experiment on the Grade Retention dataset introduced by \citet{Deaner2018Proxy}. This aims to estimate the effect of grade retention based on the score of math and reading on the long-term cognitive outcomes, in which we use scores in elementary school as a treatment-inducing proxy (Z) and cognitive test scores from Kindergarten as the outcome-inducing proxy (W). Following \citet{Mastouri2021Proximal}, we generate a synthetic "ground truth" by fitting a generalized additive model to learn a structured causal model (SCM), and a Gaussian mixture model to learn unmeasured confounder based on the learned SCM. Note, this is needed since for real-world data there is no measured ground truth.

\begin{table}[]
    \centering
    \begin{tabular}{c|c|c|c|c}
         &  DFPV & CEVAE & KPV & PMMR\\\hline
        Math & 0.023(0.001) &  0.054(0.007) & 0.043(0.000) & 0.032(0.001)\\
        Reading & 0.027(0.002) & 0.082(0.007) &0.028(0.000) & 0.022(0.000)\\
    \end{tabular}
    \caption{Results of grade retension dataset}
    \label{tab:deaner}
\end{table}

Table~\ref{tab:deaner} shows the result for this dataset.  In this setting, the performance of DFPV matches KPV and PMMR. As in the experiment described in the revious section, the setting is low-dimensional (one-dim treatment variable, three-dim treatment-inducing proxy, four-dim outcome-inducing proxy) and the generative model is smooth (the "ground truth" being a generalized additive model and a Gaussian mixture model). For these reasons, we might again expect this data to favor kernel methods, such as KPV and PMMR; nonetheless, our method matches them. DFPV again outperforms CEVAE in this setting.

\subsection{Experiments with Alternative dSprite Experiments} \label{sec:original-dsprite}
In the preceding version of this document, we employed a different experimental dSprite setting. 
We used the  structural function 
\begin{align*}
    f_{\mathrm{struct}}(A) =\frac{\|BA\|_2^2 - 5000}{1000}.
\end{align*}
where each element of the matrix $B \in \mathbb{R}^{10\times4096}$ was generated from $Unif(0.0, 1.0)$, and the outcome was generated as 
\begin{align*}
    Y = 12 (\mathtt{posY}-0.5)^2 f_{\mathrm{struct}}(A) + \varepsilon, \varepsilon\sim\mathcal{N}(0, 0.5).
\end{align*}
The generation generation process for variables $(Z,W,A)$ is unchanged in the present document. 

We now describe a limitation of this setting.\footnote{We  thank
Olawale  Salaudeen for alerting us to this issue.}
Sprite images $A$ have a small number of pixels with value $1$, and many pixels with value $0$. Each entry of $BA$ thus represents a sum of uniformly distributed independent random variables from $B$ corresponding to the nonzero entries of $A$. For this reason, the position of $A$ is very difficult to recover from $BA$, since this would require memorizing the specific sum of uniform random values for each sprite position. In practice, the structural function effectively appears as a constant function with additional noise due to $\|BA\|_2^2$.
Nonetheless, even recovery of this (effectively) constant function benefits from the proxy setting. Results of these earlier experiments are show in  Figure~\ref{fig:dsprite-original}.
\begin{figure}
    \centering
    \includegraphics[width = 0.5\columnwidth]{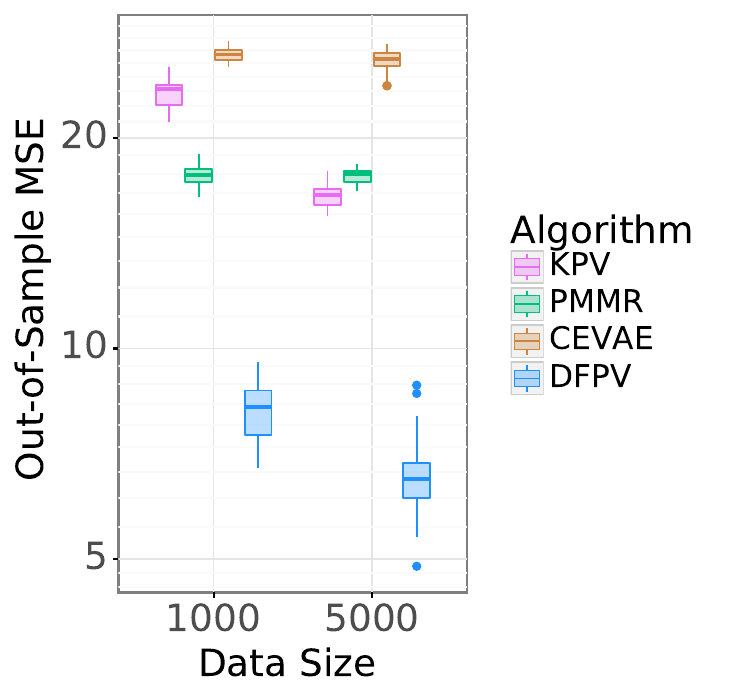}
    \caption{Alternative structural function for dsprite}
    \label{fig:dsprite-original}
\end{figure}

\section{Experiment Details}
\label{sec:data-generation-process}
In this section, we present the data generation process of experiments and the detailed settings of hyper-parameters.

\subsection{Demand Design Experiment} \label{sec:demand-design-data-generation}
Here, we introduce the details of demand design experiments. 
The observations are generated from the following causal model,
\begin{align*}
    Y = P\left(\exp\left(\frac{V-P}{10}\right) \wedge 5\right) - 5g(D) + \varepsilon, \quad \varepsilon\sim\mathcal{N}(0,1),
\end{align*}
where $Y$ represents sales, $P$ is the treatment variable (price), and these are confounded by potential demand $D$. Here we denote $a \wedge b = \min(a,b)$, and the function $g$ as 
\begin{align*}
    g(d) = 2\left(\frac{(d-5)^4}{600} + \exp(-4(d-5)^2) + \frac{d}{10} -2 \right).
\end{align*}

To correct this confounding bias, we introduce cost-shifter $C_1,C_2$ as a treatment-inducing proxy, and views $V$ of the reservation page  as the outcome-inducing proxy.
Data is sampled as 
\begin{align*}
    &D \sim \mathrm{Unif}[0,10]\\ 
    &C_1 \sim 2\sin(2D\pi/10) + \varepsilon_1\\
    &C_2 \sim 2\cos(2D\pi/10) + \varepsilon_2\\
    &V \sim 7g(D)+45 + \varepsilon_3\\
    &P = 35+(C_1+3)g(D) + C_2 + \varepsilon_4
\end{align*}
where $\varepsilon_1, \varepsilon_2, \varepsilon_3, \varepsilon_4 \sim \mathcal{N}(0,1)$. From observations of $(Y, P, C_1, C_2, V)$, we estimate $\hat{f}_\mathrm{struct}$ by PCL. For each estimated $\hat{f}_\mathrm{struct}$, we measure out-of-sample error as the mean square error of $\hat{f}$ versus true $f_\mathrm{struct}$ obtained from Monte-Carlo simulation. Specifically, we consider 10 evenly spaced values of $p \in [10, 30]$ as the test data.

\subsection{dSprite Experiment} \label{sec:dsprite-design-data-generation}

Here, we describe the data generation process for the dSprites dataset experiment. This is an image dataset parametrized via five latent variables ({\tt shape,  scale, rotation, posX} and {\tt posY}). The images are $64 \times 64 = 4096$-dimensional. In this experiment, we fixed the {\tt shape} parameter to {\tt heart}, i.e. we only used the heart-shaped images. The other latent parameters take values of $ \mathtt{scale}\in [0.5, 1],~\mathtt{rotation} \in [0, 2\pi],~\mathtt{posX} \in [0,1],~\mathtt {posY} \in [0, 1]$.

From this dataset, we generate the treatment variable $A$ and outcome $Y$ as follows:
\begin{enumerate}
    \item Uniformly samples latent parameters $(\mathtt{scale}, \mathtt{rotation}, \mathtt{posX}, \mathtt{posY})$.
    \item Generate treatment variable $A$ as
    \begin{align*}
        A = \mathtt{Fig}(\mathtt{scale}, \mathtt{rotation}, \mathtt{posX}, \mathtt{posY}) + \vec{\eta}_A.
    \end{align*}
    \item Generate outcome variable $Y$  as
    \begin{align*}
    Y =  12(\mathtt{posY}-0.5)^2\frac{(\mathrm{vec}(\vec{B})^\top A)^2 - 3000}{500} + \varepsilon, \quad \varepsilon \sim \mathcal{N}(0, 0.5).
    \end{align*}
\end{enumerate}

Here, function $\mathtt{Fig}$ returns the corresponding image for the latent parameters, and $\vec{\eta}, \varepsilon$ are noise variables generated from $\vec{\eta}_A \sim \mathcal{N}(0.0, 0.1I)$ and $\varepsilon \sim \mathcal{N}(0.0, 0.5)$. The matrix $\vec{B} \in \mathbb{R}^{64\times 64}$ was given as $B_{ij} = |32-j| / 32$. From the data generation process, we can see that $A$ and $Y$ are confounded by {\tt posY}. Treatment variable $A$ is given as a figure corrupted with Gaussian random noise. The variable $\mathtt{posY}$ is not revealed to the model, and there is no observable confounder. The structural function for this setting is
\begin{align*}
    f_{\mathrm{struct}}(a) = \frac{(\mathrm{vec}(\vec{B})^\top a)^2 - 3000}{500}.
\end{align*}

To correct this confounding bias, we set up the following PCL setting. We define the treatment-inducing variable $Z = (\mathtt{scale},\mathtt{rotation}, \mathtt{posX}) \in \mathbb{R}^3$, and the outcome-inducing variable by another figure that shares the same $\mathtt{posY}$, with the remaining latent parameters fixed as follows:
\begin{align*}
    W = \mathtt{Fig}(0.8, 0, 0.5, \mathtt{posY}) + \vec{\eta}_W,
\end{align*}
where $\vec{\eta}_W \sim \mathcal{N}(0.0, 0.1I)$.

We use 588 test points for measuring out-of-sample error, which are generated from the grid points of latent variables. The grids consist of 7 evenly spaced values for $\mathtt{posX}, \mathtt{posY}$, 3 evenly spaced values for $\mathtt{scale}$, and 4 evenly spaced values for $\mathtt{orientation}$.

\subsection{Policy Evaluation Experiments} \label{sec:ope-data-generation}
We use the same data $(Y, P, C_1, C_2, V)$ in demand design for policy evaluation experiments. We consider two policies. One is a policy depends on costs $C_1, C_2$ which is
\begin{align*}
    \pi_{C_1,C_2}(C_1, C_2) = 23 + C_1  C_2.
\end{align*}
To conduct offline-policy evaluation, we use data $(C_1,C_2,V)$ to compute the empirical average of $h(\pi_{C_1,C_2}(C_1, C_2),V)$. In our second experiment, the policy depends on current price $P$, which is given as
\begin{align*}
    \pi_P(P) = \max(0.7P, 10).
\end{align*}
Again, we use data $(P,V)$ to compute the empirical average of $h(\pi_{P}(P),V)$.

\subsection{Hyper-parameters and network architectures} \label{sec:hyper-param-and-architectures}
Here, we describe the network architecture and hyper-parameters of all experiments. 

For KPV and PMMR method, we used the Gaussian kernel where the bandwidth is determined by the median trick. We follow the procedure for hyper-parameter tuning proposed in \citet{Mastouri2021Proximal} in selecting the regularizers $\lambda_1,\lambda_2$.

For DFPV, we optimize the model using Adam \citep{Kingma2015Adam} with learning rate = 0.001, $\beta_1=0.9$, $\beta_2=0.999$ and $\varepsilon = 10^{-8}$. Regularizers $\lambda_1, \lambda_2$ are both set to 0.1 as a result of the tuning procedure described in Appendix~\ref{sec:hyper-param}. Network structure is given in Tables~\ref{dfpv-demand-network},~\ref{dfpv-dSprite-network}. \footnote{In ICLR submission, we applied ReLU activation to the final feature output, which causes numerical instability. Please refer to Appendix~\ref{sec:old-experimental-res} for the old results.}

In CEVAE, we attempt to reconstruct the latent variable $L$ from $(A,Z,W,Y)$ using a VAE. Here, we use a 20-dim latent variable $L$, whose the conditional distribution is specified as follows:
\begin{align*}
    q(L|A,Z,W,Y) = \mathcal{N}(\vec{V}_1\vec{\psi}_q(A,Z,W,Y), \mathrm{diag}(\vec{V}_2\vec{\psi}_q(A,Z,W,Y) \vee 0.1))
\end{align*}
where $\vec{\psi}_q$ is a neural net and $\vec{V}_1, \vec{V}_2$ are matrices to be learned, and we denote $\vec{a} \vee 0.1 = (\max(a_i, 0.1))_i$. Furthermore, we specify the likelihood distribution as follows:
\begin{align*}
    &p(W,Z|L) = \mathcal{N}(\vec{V}_3\vec{\psi}_{p(W,Z|L)}(L), \diag(\vec{V}_4\vec{\psi}_{p(W,Z|L) }(L) \vee 0.1))\\
    &p(A|L) = \mathcal{N}(\vec{V}_5\vec{\psi}_{p(A|L)}(L), \diag(\vec{V}_6\vec{\psi}_{p(A|L))}(L) \vee 0.1))\\
    &p(Y|A, L) = \mathcal{N}(\vec{\mu}_{p(Y|A,L)}(A,L), 0.5)
\end{align*}
Here, $\vec{\psi}_{p(W,Z|L)}, \vec{\psi}_{p(A|L)}, \vec{\mu}_{p(Y|A,L)}$ are neural networks. We provide the structure of neural nets in Table~\ref{cevae-demand-network}~and~\ref{cevae-dsprite-network}. Following the orignal work \cite{Louizos2017Causal}, we train all neural nets by Adamax \citep{Kingma2015Adam} with a learning rate of 0.01, which was annealed with an exponential decay schedule. We also performed early stopping according to the lower bound on a validation set. To predict structural function, we obtain $q(L)$ by marginalizing $q(L|A,Z,W,Y)$ by observed data $A,Z,W,Y$. We then output $\hat{f}_{\mathrm{struct}}(a) =  \expect[L\sim q(L)]{\expect[Y\sim p(Y|A=a,L)]{Y}}$.

\begingroup
\renewcommand{\arraystretch}{1.2}
\begin{table}[t]
    \caption{Network structures of DFPV for demand design experiments. For the input layer, we provide the input variable. For the fully-connected layers (FC), we provide the input and output dimensions.}
    
    \label{dfpv-demand-network}
    
    \begin{minipage}{0.48\hsize}
        \centering
        \begin{tabular}{cc}
            \multicolumn{2}{c}{\textbf{Stage 1 Treatment Feature} $\vec{\phi}_{\theta_{A(1)}}$ }
            \\ \hline
            Layer & Configuration \\ \hline
            1 & Input($P$) \\ \hline
            2 & FC(1, 32), ReLU \\ \hline
            3 & FC(32, 16), ReLU \\ \hline
            4 & FC(16, 8) \\ \hline
            \end{tabular}    
    \end{minipage}
    \hfill
    \begin{minipage}{0.48\hsize}
        \centering
        \begin{tabular}{cc}
            \multicolumn{2}{c}{\textbf{Treatment-inducing Proxy Feature} $\vec{\phi}_{\theta_{Z}}$ }
            \\ \hline
            Layer & Configuration \\ \hline
            1 & Input($C_1$, $C_2$) \\ \hline
            2 & FC(2, 32), ReLU \\ \hline
            3 & FC(32, 16), ReLU \\ \hline
            4 & FC(16, 8) \\ \hline
            \end{tabular}    
    \end{minipage}\\
    \vskip3mm
    \centering
    \begin{minipage}{0.48\hsize}
        \centering
        \begin{tabular}{cc}
            \multicolumn{2}{c}{\textbf{Stage 2 Treatment Feature} $\vec{\psi}_{\theta_{A(2)}}$ }
            \\ \hline
            Layer & Configuration \\ \hline
            1 & Input($P$) \\ \hline
            2 & FC(1, 32), ReLU \\ \hline
            3 & FC(32, 16), ReLU \\ \hline
            4 & FC(16, 8)\\ \hline
            \end{tabular}    
    \end{minipage}
    \hfill
    \begin{minipage}{0.48\hsize}
        \centering
        \begin{tabular}{cc}
            \multicolumn{2}{c}{\textbf{Outcome-inducing Proxy Feature} $\vec{\psi}_{\theta_{W}}$ }
            \\ \hline
            Layer & Configuration \\ \hline
            1 & Input($V$) \\ \hline
            2 & FC(1, 32), ReLU \\ \hline
            3 & FC(32, 16), ReLU \\ \hline
            4 & FC(16, 8) \\ \hline
            \end{tabular}    
    \end{minipage}
\end{table}
\endgroup  

\begingroup
\renewcommand{\arraystretch}{1.2}
\begin{table}[t]
    \caption{Network structures of DFPV for dSprite experiments. For the input layer, we provide the input variable. For the fully-connected layers (FC), we provide the input and output dimensions. SN denotes Spectral Normalization \citep{Miyato2018}. BN denotes Batch Normalization. }
    
    \label{dfpv-dSprite-network}
    
    \begin{minipage}{0.48\hsize}
        \centering
        \begin{tabular}{cc}
            \multicolumn{2}{c}{\textbf{Stage 1 Treatment Feature} $\vec{\phi}_{\theta_{A(1)}}$ }
            \\ \hline
            Layer & Configuration \\ \hline
            1 & Input($A$) \\ \hline
            2 & FC(4096, 1024), SN, ReLU \\ \hline
            3 & FC(1024, 512), SN, ReLU, BN \\ \hline
            4 & FC(512, 128), SN,  ReLU \\ \hline
            5 & FC(128, 32), SN, ReLU \\ \hline
            \end{tabular}    
    \end{minipage}
    \hfill
    \begin{minipage}{0.48\hsize}
        \centering
        \begin{tabular}{cc}
            \multicolumn{2}{c}{\textbf{Treatment-inducing Proxy Feature} $\vec{\phi}_{\theta_{Z}}$ }
            \\ \hline
            Layer & Configuration \\ \hline
            1 & Input($Z$) \\ \hline
            2 & FC(3, 128), ReLU \\ \hline
            3 & FC(128, 64), ReLU \\ \hline
            4 & FC(64, 32), ReLU \\ \hline
            \end{tabular}    
    \end{minipage}\\
    \vskip3mm
    \centering
    \begin{minipage}{0.48\hsize}
        \centering
        \begin{tabular}{cc}
            \multicolumn{2}{c}{\textbf{Stage 2 Treatment Feature} $\vec{\psi}_{\theta_{A(2)}}$ }
            \\ \hline
            Layer & Configuration \\ \hline
            1 & Input($A$) \\ \hline
            2 & FC(4096, 1024), SN, ReLU \\ \hline
            3 & FC(1024, 512), SN, ReLU, BN \\ \hline
            4 & FC(512, 128), SN,  ReLU \\ \hline
            5 & FC(128, 32), SN, ReLU \\ \hline
            \end{tabular}    
    \end{minipage}
    \hfill
    \begin{minipage}{0.48\hsize}
        \centering
        \begin{tabular}{cc}
            \multicolumn{2}{c}{\textbf{Outcome-inducing Proxy Feature} $\vec{\psi}_{\theta_{W}}$ }
            \\ \hline
            Layer & Configuration \\ \hline
            1 & Input($W$) \\ \hline
            2 & FC(4096, 1024), SN, ReLU \\ \hline
            3 & FC(1024, 512), SN, ReLU, BN \\ \hline
            4 & FC(512, 128), SN,  ReLU \\ \hline
            5 & FC(128, 32), SN, ReLU \\ \hline
            \end{tabular}    
    \end{minipage}
\end{table}
\endgroup

\begingroup
\renewcommand{\arraystretch}{1.2}
\begin{table}[t]
    \caption{Network structures of CEVAE for demand design experiment. For the input layer, we provide the input variable. For the fully-connected layers (FC), we provide the input and output dimensions. }
    
    \label{cevae-demand-network}
    \begin{minipage}{0.48\hsize}
        \centering
        \begin{tabular}{cc}
            \multicolumn{2}{c}{\textbf{Structure of} $\vec{\psi}_q$}
            \\ \hline
            Layer & Configuration \\ \hline
            1 & Input($P, Y, C_1, C_2, V$) \\ \hline
            2 & FC(5, 128), ReLU \\ \hline
            3 & FC(128, 64), ReLU \\ \hline
            4 & FC(64, 32),  ReLU \\ \hline
            \end{tabular}    
    \end{minipage}
    \hfill
    \begin{minipage}{0.48\hsize}
        \centering
        \begin{tabular}{cc}
            \multicolumn{2}{c}{\textbf{Structure of} $\vec{\psi}_{p(W,Z|L)}$}
            \\ \hline
            Layer & Configuration \\ \hline
            1 & Input($L$) \\ \hline
            2 & FC(20, 64), ReLU \\ \hline
            3 & FC(64, 32), ReLU \\ \hline
            4 & FC(32, 16),  ReLU \\ \hline
            \end{tabular}    
    \end{minipage}\\
    
    \vskip3mm
    \begin{minipage}{0.48\hsize}
        \centering
        \begin{tabular}{cc}
            \multicolumn{2}{c}{\textbf{Structure of} $\vec{\psi}_{p(A|L)}$}
            \\ \hline
            Layer & Configuration \\ \hline
            1 & Input($L$) \\ \hline
            2 & FC(20, 64), ReLU \\ \hline
            3 & FC(64, 32), ReLU \\ \hline
            4 & FC(32, 16),  ReLU \\ \hline
            \end{tabular}    
    \end{minipage}
    \hfill
    \begin{minipage}{0.48\hsize}
        \centering
        \begin{tabular}{cc}
            \multicolumn{2}{c}{\textbf{Structure of} $\vec{\mu}_{p(Y|A,L)}$}
            \\ \hline
            Layer & Configuration \\ \hline
            1 & Input($L, P$) \\ \hline
            2 & FC(21, 64), ReLU \\ \hline
            3 & FC(64, 32), ReLU \\ \hline
            4 & FC(32, 16),  ReLU \\ \hline
            5 & FC(16, 1) \\ \hline
            \end{tabular}    
    \end{minipage}
    \end{table}
\endgroup

\begingroup
\renewcommand{\arraystretch}{1.2}
\begin{table}[t]
    \caption{Network structures of CEVAE for dSprite experiment. For the input layer, we provide the input variable. For the fully-connected layers (FC), we provide the input and output dimensions.  SN denotes Spectral Normalization \citep{Miyato2018}. BN denotes Batch Normalization. }
    
    \label{cevae-dsprite-network}
    \begin{minipage}{0.48\hsize}
        \centering
        \begin{tabular}{cc}
            \multicolumn{2}{c}{\textbf{Structure of} $\vec{\psi}_q$}
            \\ \hline
            Layer & Configuration \\ \hline
            1 & Input($W,A,Z,Y$) \\ \hline
            2 & FC(8196, 1024), SN, ReLU  \\ \hline
            3 & FC(1024, 512), SN, ReLU, BN \\ \hline
            4 & FC(512, 128), SN,  ReLU \\ \hline
            5 & FC(128, 32), SN, ReLU \\ \hline
            \end{tabular}    
    \end{minipage}
    \hfill
    \begin{minipage}{0.48\hsize}
        \centering
        \begin{tabular}{cc}
            \multicolumn{2}{c}{\textbf{Structure of} $\vec{\psi}_{p(W,Z|L)}$}
            \\ \hline
            Layer & Configuration \\ \hline
            1 & Input($L$) \\ \hline
            2 & FC(20, 64), ReLU \\ \hline
            3 & FC(64, 128), ReLU \\ \hline
            4 & FC(128, 256),  ReLU \\ \hline
            \end{tabular}    
    \end{minipage}\\
    
    \vskip3mm
    \begin{minipage}{0.48\hsize}
        \centering
        \begin{tabular}{cc}
            \multicolumn{2}{c}{\textbf{Structure of} $\vec{\psi}_{p(A|L)}$}
            \\ \hline
            Layer & Configuration \\ \hline
            1 & Input($L$) \\ \hline
            2 & FC(20, 64), ReLU \\ \hline
            3 & FC(64, 128), ReLU \\ \hline
            4 & FC(128, 256),  ReLU \\ \hline
            \end{tabular}    
    \end{minipage}
    \hfill
    \begin{minipage}{0.48\hsize}
        \centering
        \begin{tabular}{cc}
            \multicolumn{2}{c}{\textbf{Structure of} $\vec{\mu}_{p(Y|A,L)}$}
            \\ \hline
            Layer & Configuration \\ \hline
            1 & Input($L, A$) \\ \hline
            2 & FC(4116, 1024), SN, ReLU, BN \\ \hline
            3 & FC(1024, 512), SN, ReLU, BN \\ \hline
            4 & FC(512, 128), SN,  ReLU \\ \hline
            5 & FC(128, 32), SN, ReLU \\ \hline
            6 & FC(32, 1) \\ \hline
            \end{tabular}    
    \end{minipage}
    \end{table}
\endgroup

\section{Old Experimental Results}
\label{sec:old-experimental-res}
In our Neurips submission\footnote{\url{https://proceedings.neurips.cc/paper/2021/file/dcf3219715a7c9cd9286f19db46f2384-Paper.pdf}}, we apply another ReLU activation to each features. However, this may suffer from numerical instability they can all be zero during in the training, as shown in Figure~\ref{fig:demand-exp}.

\begin{figure}
    \centering
    \includegraphics[width=0.8\textwidth]{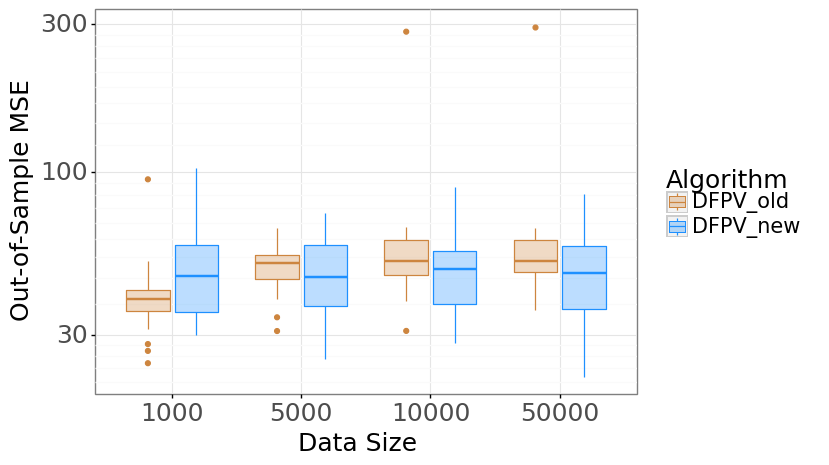}
    \caption{Performance comparison between DFPV\_old (w. ReLU activation) and DFPV\_new (w.o. ReLU activation) in Demand Experiment Results}
    \label{fig:demand-exp}
\end{figure}

\end{document}